\pdfoutput=1
\documentclass[runningheads]{llncs}
\usepackage{graphicx}
\usepackage{amsmath,amssymb} 
\usepackage{color}
\usepackage[width=122mm,left=12mm,paperwidth=146mm,height=193mm,top=12mm,paperheight=217mm]{geometry}
\usepackage{times}
\usepackage{epsfig}
\usepackage{graphicx}
\usepackage{amsmath}
\usepackage{amssymb}
\usepackage{epstopdf}
\usepackage{multirow}
\usepackage{subfig}
\usepackage{algorithm}
\usepackage{algorithmic}
\usepackage{cite}
\usepackage{tabularx, booktabs}
\usepackage{arydshln}

\newcolumntype{C}[1]{>{\centering\let\newline\\\arraybackslash\hspace{0pt}}m{#1}}
\newcolumntype{L}[1]{>{\raggedright\let\newline\\\arraybackslash\hspace{0pt}}m{#1}}

\usepackage{fancyheadings}

\usepackage{array,url,kantlipsum}

\newcommand{\Mark}[1]{\textsuperscript{#1}}

\usepackage[hang,flushmargin]{footmisc}

\newcolumntype{Y}{>{\centering\arraybackslash}X}

\newcommand\norm[1]{\left\lVert#1\right\rVert}

\begin{document}
\pagestyle{headings}
\mainmatter
\def\ECCV18SubNumber{1106}  

\if false
\title{Incremental Non-Rigid Structure-from-Motion with Unknown Focal Length} 

\titlerunning{ECCV-18 submission ID 1106}

\authorrunning{ECCV-18 submission ID 1106}

\author{Anonymous ECCV submission}
\institute{Paper ID 1106}
\maketitle
\fi
\title{Incremental Non-Rigid Structure-from-Motion \\ with Unknown Focal Length } 
\titlerunning{Incremental Non-Rigid Structure-from-Motion with Unknown Focal Length}

\authorrunning{Thomas Probst, Danda Pani Paudel, Ajad Chhatkuli,  and Luc Van Gool}

\author{Thomas Probst\Mark{1}, Danda Pani Paudel\Mark{1}, Ajad Chhatkuli\Mark{1}, and Luc Van Gool\Mark{1,2}}


\institute{	\Mark{1}\,Computer Vision Lab, ETH Z\"urich, Switzerland \\
		\Mark{2}\,VISICS, ESAT/PSI, KU Leuven, Belgium \\
	\email{ \{probstt,paudel,ajad.chhatkuli,vangool\}@vision.ee.ethz.ch}}

\maketitle


\begin{abstract}
The perspective camera and the isometric surface prior have recently gathered increased attention for Non-Rigid Structure-from-Motion (NRSfM). Despite the recent progress, several challenges remain, particularly the computational complexity and the unknown camera focal length. In this paper we present a method for incremental Non-Rigid Structure-from-Motion (NRSfM) with the perspective camera model and the isometric surface prior with unknown focal length. In the template-based case, we provide a method to estimate four parameters of the camera intrinsics. For the template-less scenario of NRSfM, we propose a method to upgrade reconstructions obtained for one focal length to another based on local rigidity and the so-called Maximum Depth Heuristics (MDH). On its basis we propose a method to simultaneously recover the focal length and the non-rigid shapes. We further solve the problem of incorporating a large number of points and adding more views in MDH-based NRSfM and efficiently solve them with Second-Order Cone Programming (SOCP). This does not require any shape initialization and produces results orders of times faster than many methods. We provide evaluations on standard sequences with ground-truth and qualitative reconstructions on challenging YouTube videos. These evaluations show that our method performs better in both speed and accuracy than the state of the art.
\end{abstract}
%
%

\section{Introduction}
\label{sec:intro}
Given images of a rigid object from different views, Structure-from-Motion (SfM)~\cite{Longuet1981,Nister2004,Hartley2004} allows the computation of the object's 3D structure. However, many such objects of interest are non-rigid and the rigidity constraints of SfM do not hold. The ever increasing number of monocular videos with deforming objects means provides a large incentive for being able to reconstruct such scenes. Such reconstruction problems can be solved with Non-Rigid Structure-from-Motion (NRSfM) which uses multiple images of a deforming object to reconstruct its 3D from a single camera. Another related approach computes the shape based on the object's template shape and its deformed image, also termed as Shape-from-Template (SfT). While SfM is well-posed and has already seen several applications in commercial software~\cite{Photoscan,Autodesk}, non-rigid reconstruction has inherent theoretical problems. It is severely under-constrained without prior knowledge of the deformation or the shapes. In fact given a number of images, infinite possibilities of deformations exist that provide the same image projections. Therefore, one of the major challenges in NRSfM is to efficiently combine a realistic deformation constraint and the camera projection model to reduce the solution ambiguity.
\begin{figure}[t]
\small
\centering
  \begin{minipage}{1\textwidth}
  \centering
   \includegraphics[width=0.3\textwidth]{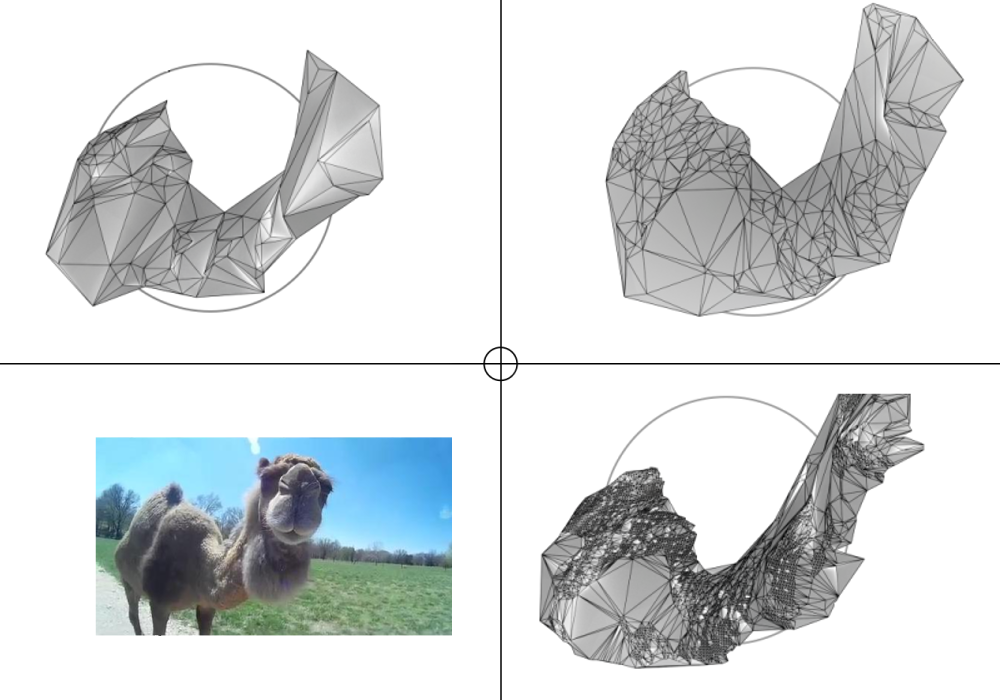}
   \hspace{0.2cm}
  \includegraphics[width=0.3\textwidth]{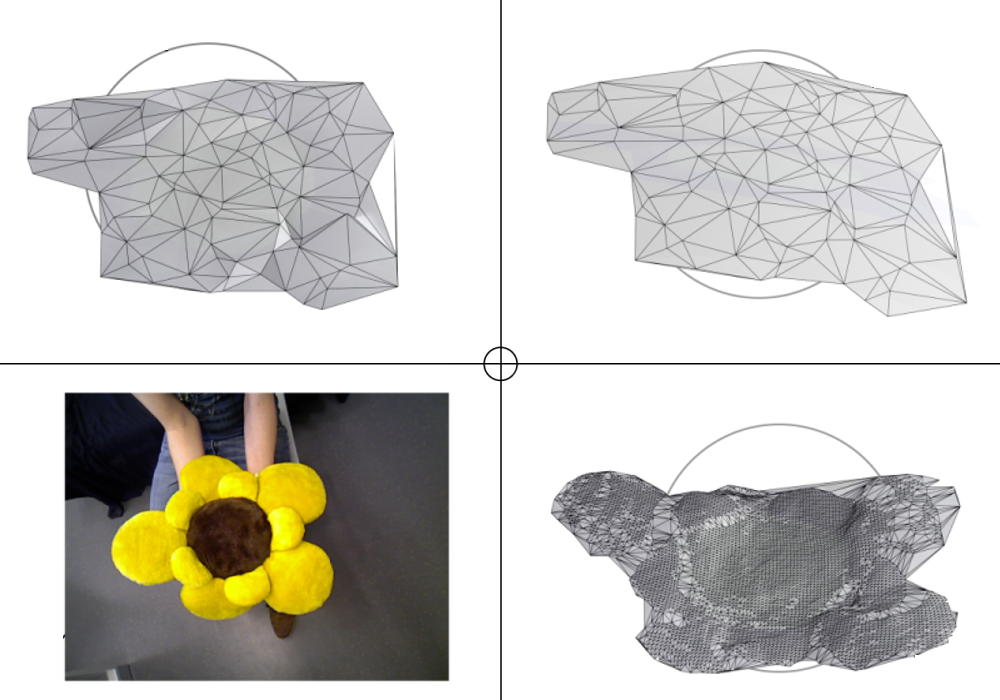}
  \hspace{0.2cm}
  \includegraphics[width=0.3\textwidth]{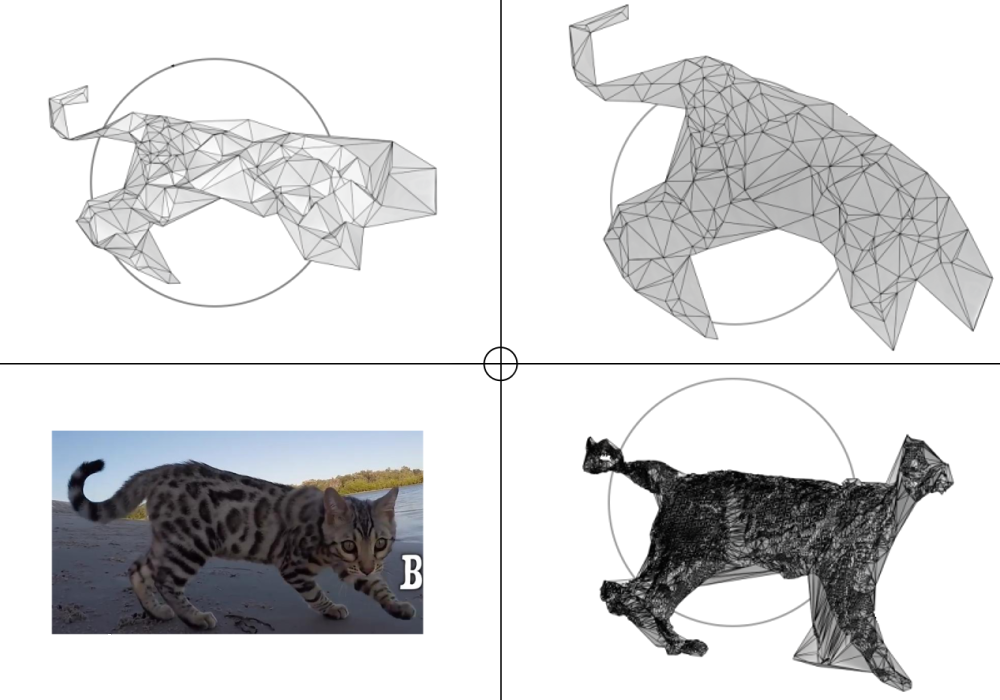}
  \end{minipage}
\caption{\textbf{Qualitative Results.} Comparison of our dense NRSfM method (bottom-right) to Ji~et~al.~\cite{Ji2017} (top-left) and Dai~et~al.~\cite{Dai2012} (top-right) on three different sequences. 
} 
\label{fig:denseQualitativeResultsFrontPage}
\end{figure}

A large majority of previous methods tackle NRSfM with an affine camera model and a low rank approximation of the deforming shapes~\cite{Bregler2000,Torresani2008,DelBue2008,Dai2012,Garg2013,Fayad2010,Agudo2014bmvc,Taylor2010}. However, such methods do not handle perspective effects and nonlinear deformations very well. In this paper we study the use of the uncalibrated perspective camera and the isometric deformation prior for non-rigid reconstruction. Isometry is a geometric prior which implies that the geodesic distances on the surface are preserved with the deformations. This is a good approximation for many real objects such as a human body, paper-like surfaces, or cloth. In SfT, the use of the isometric deformation prior with the perspective camera is considered to be the state-of-the-art~\cite{Bartoli2013iccv,Ngo2016,ChhatkuliPAMI2016} among the parameter-free approaches.  In particular, ~\cite{Bartoli2013,Bartoli2013iccv} also estimate the focal length while recovering the deformation. In NRSfM, some recent methods~\cite{Chhatkuli2016,Ji2017} provide a convex formulation with the inextensible deformation for a calibrated perspective camera setup. The reconstruction is achieved by maximizing depth along the sightlines introduced in ~\cite{Perriollat2011,Salzmann2011} for template-based reconstruction. Although the methods use the perspective camera model and geometric priors for non-rigid reconstruction, their computational complexity does not allow reconstructing a large number of points. On the other hand, some recent dense methods using the perspective camera model have shown promising results, but they rely on piecewise rigidity constraints~\cite{Kumar2017,Russell2014} and shape initialization; this may be too constraining for several applications. Furthermore, methods using the perspective camera either rely on known intrinsics or cannot handle significant nonrigidity~\cite{XiaoKanade2005}. To the best of our knowledge, estimation of the unknown focal length has not been investigated in NRSfM for deforming surfaces.

In this paper we address the aforementioned issues with methods based on the convex relaxation of isometry. More precisely, we provide the following contributions: \emph{a)} a method to `upgrade' the non-rigid reconstruction obtained using incorrect camera intrinsics to the reconstruction of the correct one, \emph{b)} a method to estimate  intrinsics - all five entries in the case of SfT and the unknown focal length in the case template-less NRSfM \emph{c)} an incremental method to add more points to the sparse 3D point-sets for consistent and semi-dense reconstruction \emph{d)} online method of reconstruction by adding images. Besides being of immense practical concern and theoretical value, questions \emph{a)} and \emph{b)} have not been attempted for NRSfM for deforming objects. 
We provide a unified framework to solve the problems \emph{a)} through \emph{d)} using depth maximization and the relaxations of the isometry prior. We provide theoretical justification along with practical methods for intrinsics/focal length estimation as well as densification and online reconstruction strategies. Despite being extremely challenging, we show the applicability of our method with  compelling results. A few examples among them  is shown in Fig~\ref{fig:denseQualitativeResultsFrontPage}. 

\subsection{Related Work}
\label{sec:relatedwork}
We discuss briefly the methods based on the isometry prior and the perspective camera model. This has been widely explored in the template-based methods~\cite{Perriollat2011,Salzmann2011,Bartoli2015}.
In particular, \cite{Salzmann2011} uses the inextensibility as a relaxation of the isometry prior in order to formulate non-rigid reconstruction as a convex problem by maximizing the depth point-wise. Several recent NRSfM methods~\cite{Vicente2012,Chhatkuli2014,Chhatkuli2016,Parashar2016,Ji2017} also use isometry or inextensibility with the perspective camera model. \cite{Chhatkuli2014,Parashar2016} require the correspondence mapping function with its first and second-order derivatives limiting their application in practice.
\cite{Chhatkuli2016} improved upon \cite{Vicente2012} by providing a convex solution to NRSfM. They achieve this by maximizing pointwise depth in all views under the inextensibility cone constraints of \cite{Salzmann2011} while also computing the template geodesics. Very recently a method \cite{Ji2017} improving upon \cite{Chhatkuli2016} suggested the use of maximization of sightlines rather than the pointwise depth. Both these methods have shown that moving the surface away from the camera under the inextensibility constraints can be formulated as a convex problem effectively reconstructing non-rigid as well as rigid objects. A different class of methods that use energy minimization approach on an initial solution also use the perspective camera model but with a piece-wise rigidity prior~\cite{Kumar2017,Russell2014}. However, all of these methods discussed here require the calibrated camera for reconstruction and do not provide any insights on how they can be extended to an uncalibrated camera. One notable exception is given by~\cite{XiaoKanade2005}, however this approach is limited to dynamic scenes featuring a few independently moving objects~\cite{Salzmann2007,Akhter2011}. Yet another problem that has not been addressed in \cite{Chhatkuli2016,Ji2017} is the incremental reconstruction of a large number of points. Semi-dense or dense reconstruction as such is not possible here due to the high computational complexity of these methods.
\section{Problem Modelling}
We pose the NRSfM problem as that of finding point-wise depth in each view. We write the unknown depth as $\lambda_i^l$ and the known homogeneous image coordinates as $\mathsf{u}_i^l$, for the  point $i$ in the $l$-th image. A set of neighboring points of $i$ is denoted by  $\mathcal{N}(i)$. $d_{ij}$ represents the template geodesic distance between point $i$ and $j$, which is an unknown quantity for the NRSfM problem and a known quantity for the SfT problem. We define a nearest neighborhood graph as a set of fixed number of neighbors for each point $i$~\cite{Chhatkuli2016}. To represent the exact isometric NRSfM problem, we also introduce a geodesic distance function between two 3D points on the surface $\mathcal{S}$, $g_\mathcal{S}(x,y): \mathbb{R}^3 \times \mathbb{R}^3 \to \mathbb{R}$. Given the camera intrinsics $\mathsf{K}$, the isometric NRSfM problem can be written as:
\begin{align}
\label{eq:iso-nrsfm}
\begin{split}
& \text{Find}\quad \mathsf{K},\ \lambda_i^l \\
& \text{s.t.}\quad  g\left(\mathsf{K}^{-1}\lambda_{i}^l\mathsf{u}_{i}^l,\ \mathsf{K}^{-1}\lambda_{j}^l\mathsf{u}_{j}^l\right)=  d_{ij}, \ \forall i, \forall j.
\end{split}
\end{align}
~\eqref{eq:iso-nrsfm} defines a non-convex problem and is also not tractable in its given form. It has been shown that with various relaxations~\cite{Vicente2012,Chhatkuli2016,Ji2017}, problem~\eqref{eq:iso-nrsfm} can be solved for a known $\mathsf{K}$ when different views and deformations are observed. In order to tackle the NRSfM problem with an unknown focal length we start with the observation that not all such solutions provide isometrically consistent shapes through all the views. We formulate our methods in the following sections.

\section{Uncalibrated NRSfM}
Given a known object template and a calibrated camera the NRSfM problem in ~\eqref{eq:iso-nrsfm} can be formulated as a convex problem by relaxing the isometry constraint with an inextensibility constraint~\cite{Salzmann2011} as below:
\begin{equation}
\begin{aligned}
& \underset{\lambda_{i}^l}{\textnormal{max}}
& & \sum_{l} \sum_{i} \lambda_{i}^l, \\
& \textnormal{~s.t.}
& & \norm{\mathsf{K}^{-1}(\lambda_{i}^l\mathsf{u}_{i}^l - \lambda_{j}^l\mathsf{u}_{j}^l)}\leq  d_{ij}, & \forall j\in \mathcal{N}(i).
\end{aligned}
\label{eq:basicNRSfM}
\end{equation}
We are, however, interested on solving the same problem when both $d_{ij}$ and $\mathsf{K}$ are unknown. Unfortunately, this problem is not only non-convex, but also unbounded. Therefore, we use two extra constraints on the variables $\mathsf{K}$ and $d_{ij}$ such that the problem of~\eqref{eq:basicNRSfM}, for unknown $d_{ij}$ and $\mathsf{K}$, becomes bounded.
\begin{equation}
\begin{aligned}
\sum_{i}\sum_{j\in\mathcal{N}(i)}d_{ij}=1,&& \mathsf{K}\leq\overline{\mathsf{K}}.
\end{aligned}
\label{eq:dAndKConst}
\end{equation}
Despite being bounded with the addition of~\eqref{eq:dAndKConst}, the reconstruction problem is still non-convex. More importantly, the maximization of the objective function favors the solution when $\mathsf{K}$ is as close as possible to $\overline{\mathsf{K}}$. Therefore, we instead solve the reconstruction problem in ~\eqref{eq:basicNRSfM} with a fixed initial guess $\mathsf{\hat{K}}$ and seek for the upgrade of both intrinsics and reconstruction later. Note that fixing the intrinsics makes the problem convex and identical to that in \cite{Chhatkuli2016}. 
\begin{equation}
\begin{aligned}
& \underset{\lambda_{i}^l, d_{ij}}{\textnormal{max}}
& & \sum_{l} \sum_{i} \lambda_{i}^l, \\
& \textnormal{~s.t.}
& & \norm{\mathsf{\hat{K}}^{-1}(\lambda_{i}^l\mathsf{u}_{i}^l - \lambda_{j}^l\mathsf{u}_{j}^l)}\leq  d_{ij},& j\in \mathcal{N}(i),\\
&&&\sum_{i}\sum_{j\in\mathcal{N}(i)}d_{ij}=1.
\end{aligned}
\label{eq:withBoundNRSfM}
\end{equation}
Now, we are interested to upgrade the solution of~\eqref{eq:withBoundNRSfM} such that the upgraded reconstruction correctly describes the deformed object in the 3D-space. In this work, the upgrade is carried out using a pointwise upgrade equation. In the following, we first derive this upgrade equation assuming that the correct focal length is known and then provide the theory and practical approaches to recover the unknown focal length.

\subsection{Upgrade Equation}
Let us consider, $\lambda_i^l$ and  $\hat{\lambda}_i^l$ are depths, of the point represented by $\mathsf{u}_i^l$, obtained from~\eqref{eq:basicNRSfM} and~\eqref{eq:withBoundNRSfM}, respectively. The following proposition is the key ingredient of our work that relates  $\hat{\lambda}_i^l$ to $\lambda_i^l$ for the reconstruction upgrade.

\begin{proposition}
For $\mathsf{u}_i^l\approx \mathsf{u}_{\mathcal{N}(i)}^l$, $\hat{\lambda}_i^l$ can be upgraded to $\lambda_i^l$ with the known $\mathsf{K}$ using,
\begin{equation}
\lambda_i^l \approx \frac{\hat{\lambda}_i^l\norm{\mathsf{\hat{K}^{-1}}\mathsf{u}_i^l}}{\norm{\mathsf{\mathsf{K}^{-1}}\mathsf{u}_i^l}}.
\label{eq:upgradeEq}
\end{equation}
\label{prop:upgradeLambda}
\end{proposition}
\begin{proof}
It is sufficient to show that every~$j\in\mathcal{N}(i)$ satisfies
$\norm{\mathsf{\hat{K}}^{-1}(\hat{\lambda}_{i}^l\mathsf{u}_{i}^l - \hat{\lambda}_{j}^l\mathsf{u}_{j}^l)} \approx \norm{\mathsf{K}^{-1}(\lambda_{i}^l\mathsf{u}_{i}^l - \lambda_{j}^l\mathsf{u}_{j}^l)}$. From~\eqref{eq:upgradeEq}, for any $u_i^l\approx u_{\mathcal{N}(i)}^l$, $\norm{\mathsf{\hat{K}}^{-1}(\hat{\lambda}_{i}^l\mathsf{u}_{i}^l - \hat{\lambda}_{j}^l\mathsf{u}_{j}^l)}^2$ can be expressed as,
\begin{equation}
\begin{aligned}
\approx&\norm{\mathsf{\mathsf{K}^{-1}}\mathsf{u}_i^l}^2 \norm{\mathsf{\hat{K}}^{-1}(\lambda_{i}^l - \lambda_{j}^l)\mathsf{u}_{i}^l}^2\mathbin{/}\norm{\mathsf{\hat{K}^{-1}}\mathsf{u}_i^l}^2,\\
=&(\lambda_{i}^l - \lambda_{j}^l)^2\norm{\mathsf{\mathsf{K}^{-1}}\mathsf{u}_i^l}^2 \approx \norm{\mathsf{K}^{-1}(\lambda_{i}^l\mathsf{u}_{i}^l - \lambda_{j}^l\mathsf{u}_{j}^l)}^2. \hfill\qed
\end{aligned}
\end{equation}
\end{proof}
Note that the condition $\mathsf{u}_i^l\approx\mathsf{u}_{\mathcal{N}(i)}^l$ is valid for any two sufficiently close neighbors. Such neighbors can be chosen using only the image measurements. More importantly, the assumption $\mathsf{u}_i^l\approx\mathsf{u}_{\mathcal{N}(i)}^l$ still allows depths $\lambda_i^l$ and $\lambda_{\mathcal{N}(i)}^l$ to be different. This plays a vital role especially when the close neighboring points differ distinctly in depth, either due to camera perspective or high frequency structural changes. Although,~\eqref{eq:upgradeEq} is only a close approximation for the reconstruction upgrade, its upgrade quality in practice was observed to be accurate.  
The following remark concerns Proposition~\ref{prop:upgradeLambda}.
\begin{remark}
As the guess on intrinsics $\mathsf{\hat{K}}$ tends to the real intrinsics $\mathsf{K}$, the  upgrade equation~\eqref{eq:upgradeEq} holds true for exact equality even when $\mathsf{u}_i^l\not\approx \mathsf{u}_{\mathcal{N}(i)}^l$. In other words,
\begin{equation}
\lim_{\mathsf{\hat{K}}\to\mathsf{K}} \lambda_i^l = \hat{\lambda}_i^l.
\end{equation}
\end{remark}

\subsection{Upgrade Strategies}
The upgrade equation presented in Proposition~\ref{prop:upgradeLambda} assumes that the exact intrinsics $\mathsf{K}$ is known. However, for uncalibrated NRSfM, $\mathsf{K}$ is unknown. While the principal point can be assumed to be at the center of the image for most cameras~\cite{nister2004untwisting}, nothing can be said about the focal length. We henceforth, present strategies to estimate $\mathsf{K}$ in two different scenarios of known and unknown shape template. We rely on the fact that isometric deformation, to a large extent, preserves local rigidity. This is reflected somewhat in the reconstruction obtained from \eqref{eq:withBoundNRSfM}. However, due to changes in the perspective and the extension of points along incorrect sightlines, the use of incorrect intrinsics produces reconstructions that are very less likely to remain isometric across different views. Similarly, an upgrade towards the correct intrinsics in that case produces reconstructions which satisfy the isometry better. This is also supported by the results in Section \ref{sec:results}.
There are various ways one can use isometry of the reconstructed surfaces to determine the correct intrinsics. A very simple method would be to use the fact that given reconstructed points that are dense enough, the correct intrinsics must preserve the local euclidean distance. For $\hat{a}_{i}=\hat{\lambda}_i\norm{\mathsf{\hat{K}}^{-1}\mathsf{u}_i}$, the euclidean distance between two upgraded neighboring 3D points, in any view as a function of intrinsics, can be expressed as,

\begin{equation}
\hat{d}_{ij}(\mathsf{K})=
\norm{
\frac{\hat{a}_i\mathsf{K}^{-1} \mathsf{u}_{i} }{\norm{\mathsf{\mathsf{K}^{-1}}\mathsf{u}_i}}
- 
\frac{\hat{a}_j\mathsf{K}^{-1} \mathsf{u}_{j} }{\norm{\mathsf{\mathsf{K}^{-1}}\mathsf{u}_j}}
}.
\label{eq:upgradedDist}
\end{equation}
Now, we present techniques to estimate $\mathsf{K}$ when the shape template is known (SfT), followed by a method to estimate the focal length for template-less case of NRSfM.

\subsubsection{Template-based Calibration}
For the sake of simplicity, we present the calibration theory using only one image. This is also the sufficient condition for reconstruction when the shape template is known~\cite{Salzmann2011}. Recall that for SfT,
$d_{ij}$ in~\eqref{eq:withBoundNRSfM} are already known during the reconstruction process. For known template distance $d_{ij}$ and the estimated euclidean distance after reconstruction upgrade $\hat{d}_{ij}(\mathsf{K})$, the intrinsics $\mathsf{K}$ can be estimated by minimizing,
\begin{equation}
\label{eq:Phi_T}
\Phi_{\scriptscriptstyle T}(\mathsf{K}) = \sum_i\sum_{j\in\mathcal{N}(i)}\Big(d_{ij}-\hat{d}_{ij}(\mathsf{K})\Big)^2.
\end{equation}
Alternatively, one can also derive polynomial equations on the entries of the so-called Image of the Absolute Conic (IAC), defined as $\mathsf{\Omega}=\mathsf{K^{-\intercal}}\mathsf{K^{-1}}$.  
\begin{proposition}
As long as the rigidity between any pair $\{\mathsf{u}_i, \mathsf{u}_j\}$ is maintained,  either for any $\mathsf{\hat{K}}$ and $\mathsf{u}_i\approx\mathsf{u}_j$ or for  any pair $\{\mathsf{u}_i, \mathsf{u}_j\}$ as $\mathsf{\hat{K}}\rightarrow \mathsf{K}$, the IAC can be approximated by solving,
\begin{equation}
\mathsf{u}_i^\intercal\mathsf{\Omega}\mathsf{u}_i\mathsf{u}_j^\intercal\mathsf{\Omega}\mathsf{u}_j=\gamma_{ij}\big(\mathsf{u}_i^\intercal\mathsf{\Omega}\mathsf{u}_j\big)^2,
\label{eq:templateDIAC}
\end{equation}
for sufficiently many pairs, where,
\begin{equation}
\gamma_{ij} = \Big(\frac{2\hat{a}_{i}\hat{a}_j}{\hat{a}_i^2+\hat{a}_j^2-d_{ij}^2}\Big)^2.
\label{eq:gammaDefTemplate}
\end{equation}
\end{proposition}
We provide the proof in the supplementary material.

Note that~\eqref{eq:templateDIAC} is a degree 2 polynomial on the entries of $\mathsf{\Omega}$. Since, $\mathsf{\Omega}$ has 5 degrees of freedom, it can be estimated from 5 pairs of image points, using numerical methods. 
 
The core idea of our template-based calibration consists of three steps: (i) a  fixed number of hypothesis generation, (ii) hypothesis validation using the upgraded reconstruction  quality, (iii) refinement of the best hypothesis.

\noindent\textbf{Hypothesis generation:} Given the template-based uncalibrated reconstruction from~\eqref{eq:withBoundNRSfM}, we generate a set of hypotheses for camera intrinsics from randomly selected sets of minimal closest-point pairs. For every minimal set, we solve~\eqref{eq:templateDIAC} for $\mathsf{\Omega}$ to obtain these hypotheses. Then, the camera intrinsics $\mathsf{K}$ is recovered by performing the Cholesky-decomposition on $\mathsf{\Omega}$.

\noindent\textbf{Hypothesis validation:} Each hypothesis is validated by computing its 3D reconstruction error. To do so, we  first upgrade the initial reconstruction using the upgrade~\eqref{eq:upgradeEq} for current hypothesis. Then, the reconstruction error is computed using~\eqref{eq:Phi_T}. The hypothesis that results into minimum reconstruction error is chosen for further refinement.   

\noindent\textbf{Intrinsics refinement:} Starting from the best hypothesis, we refine the intrinsics by minimizing the following objective function:

\begin{equation}\label{eq:templateCalibCost}
\mathcal{E}(\mathsf{K})=\Phi_T(\mathsf{K})+ k_{(1,3)}^2 + k_{(2,3)}^2 + \Big(1- \frac{k_{(1,1)}}{k_{(2,2)}}\Big)^2,
\end{equation}
where, $k_{(i,j)}$ is the $i^{th}$-row and $j^{th}$-column entry of the normalized intrinsic matrix $\mathsf{K}$. Note that, we regularize the 3D reconstruction 
error $\Phi_T(\mathsf{K})$ by the expected structure of $\mathsf{K}$ (i.e. principal point close to the center and unit aspect ratio). Our regularization term is often the main objective for existing autocalibration methods~ \cite{nister2004untwisting,chandraker2007autocalibration}.
The minimization of objective $\mathcal{E}(\mathsf{K})$ can be carried out efficiently using locally optimal iterative refinement methods.

Now, we summarize our calibration method in Algo.~\ref{alg:templateCalibration}.
\begin{algorithm}[h]
{\fontsize{8pt}{8pt}\selectfont
\caption{\small [$\mathsf{K}$] = calibrateWithTemplate($\mathsf{\hat{K}}$)}
\label{alg:templateCalibration}
\begin{algorithmic}
 \STATE 1. Reconstruct 3D using~\eqref{eq:withBoundNRSfM} for known $d_{ij}$  and the guess $\mathsf{\hat{K}}$.
 \STATE 2.  Select multiple sets of minimal closest-point pairs  
 $\{\mathsf{u}_i, \mathsf{u}_j\}$.
 \STATE 3. For each set,  \\
 \hspace{1mm}(i) Generate hypothesis $\mathsf{\tilde{K}}$ by solving ~\eqref{eq:templateDIAC}.\\ 
 \hspace{1mm}(ii) Upgrade the reconstruction for $\mathsf{\tilde{K}}$  using~\eqref{eq:upgradeEq}.\\
 \hspace{1mm}(iii) Compute the reconstruction error for $\mathsf{\tilde{K}}$  using~\eqref{eq:Phi_T}.  
  \STATE 4.  Among all sets, choose $\mathsf{\tilde{K}}$ with best reconstruction error.  
 \STATE 5.  Refine the best hypothesis $\mathsf{\tilde{K}}$  using~\eqref{eq:templateCalibCost} to obtain  $\mathsf{K}$.
\end{algorithmic}
}
\end{algorithm}

\subsubsection{Template-less Calibration}
As the self-calibration with the unknown template is extremely challenging, we relax it by considering that the principal point is at the center of the image and that the two focal lengths are equal. We 
assume that the intrinsics are constant across views. We then measure the consistency of the upgraded local euclidean distances, defined by~\eqref{eq:upgradedDist}, across different views.
More precisely, we wish to estimate the focal length in $\mathsf{K}$ by minimizing the following objective function, 
\begin{equation}
\label{eq:Phi}
\Phi(\mathsf{K}) = \sum_{k}\sum_{l\neq k}\sum_i\sum_{j\in\mathcal{N}(i)}\Big(\hat{d}_{ij}^k(\mathsf{K})-\hat{d}_{ij}^l(\mathsf{K})\Big)^2.
\end{equation}
Ideally, it is also possible to derive polynomials on $\mathsf{\Omega}$, analogous to~\eqref{eq:templateDIAC}. This can be done by eliminating the unknown variable $d_{ij}$ from two equations for two views of the same pair. Unfortunately, the equation derived in this manner does not turn out to be easily tractable. Alternatively, one can also attempt to solve the polynomials without eliminating  variables $d_{ij}$ -- on both variables $\mathsf{\Omega}$ and $d_{ij}$.  
However for practical reasons\footnote{ For most of the cameras, it is safe to assume that their intrinsics have no skew, unit aspect ratio, and a principal point close to the image center.}, 
we design a method assuming only one entry of  $\mathsf{\Omega}$, corresponding to the focal length, is unknown. Under such assumption, we show in the supplementary materials that a polynomial of degree 4, one variable, equivalent to~\eqref{eq:templateDIAC}, can also be derived. 

In this paper, we avoid making hypothesis on the focal length, since it is not really necessary. 
Unlike the case of template-based calibration, we address the problem of template-less calibration iteratively in two steps: (i) focal length refinement, (ii) focal length validation. Henceforth for the template-less calibration, we make a slight abuse of notation by using $\mathsf{K}$ even for the intrinsics with only unknown focal length, unless mentioned otherwise.


\noindent\textbf{Focal length refinement:} Given an initial guess on focal length, its refinement is carried out by minimizing the objective function $\Phi(\mathsf{K})$ of~\eqref{eq:Phi} (optionally, on the full intrinsics). This refinement process finds a refined  $\mathsf{K}$ which results a better isometric consistency of the reconstructions across views.

\noindent\textbf{Focal length validation:} The main problem of template-less calibration is to obtain the 
validity for the given pair of intrinsics and the reconstruction. In other words, if one is given all reconstructions from all possible focal lengths, it is not trivial to know the correct reconstruction. Especially when reconstructing using overestimated intrinsics with MDH, $\mathsf{K}$ allows the average depths to dominate the objective, while preserving the isometry. This usually leads to a flat and small scaled reconstruction~\cite{ChhatkuliPAMI2016}.
Therefore an overestimated guess $\mathsf{\hat{K}}$ favors its own reconstruction over any upgraded one, while minimizing $\Phi(\mathsf{K})$. 
Relying on this observation, \emph{we seek for the isometrically consistent reconstruction with the smallest focal length}, which works very well in practice. An algebraic analysis of our reasoning is provided in the supplementary material. 

While searching for focal length, we use a sweeping procedure. On the one hand, if a reconstruction with the given focal length does not favor any upgrade, the sweeping is performed towards the lower focal length with a predefined step size, unless it starts favoring the upgrade. On the other hand, if the reconstruction favors the upgrade, we follow the suggested focal length update, until it suggests no more upgrade. The sought focal length is the one below which the upgrade is favorable, whereas above which it is not. Let $\delta(\mathsf{K}_1,\mathsf{K}_2)$  be gap in focal lengths of two intrinsics  
$\mathsf{K}_1$ and $\mathsf{K}_2$, $\Delta\mathsf{K}$ be a small step size which when added to an intrinsic matrix $\mathsf{K}$ increases its focal length by that step size. Our template-less calibration method is summarized in Algo.~\ref{alg:calibration}. 

\begin{algorithm}[h]
{\fontsize{8pt}{8pt}\selectfont
\caption{\small [$\mathsf{K}$] = calibrateWithoutTemplate($\mathsf{\hat{K}}$)}
\label{alg:calibration}
\begin{algorithmic}
\STATE 0. Set sweep direction $flag = 0$.
 \STATE 1. Reconstruct 3D using~\eqref{eq:withBoundNRSfM} for the guess $\mathsf{\hat{K}}$.
 \STATE 2.  Starting from $\mathsf{\hat{K}}$, minimize $\Phi(\mathsf{K})$ in~\eqref{eq:Phi} to obtain  $\mathsf{K}^*$.
 \STATE 3. IF $\delta(\mathsf{K}^*,\mathsf{\hat{K}})\leq\epsilon$,\\
           \hspace{10mm} IF $flag == 0$, set $\mathsf{\hat{K}}=\mathsf{K}^*-\Delta \mathsf{K}$ and goto step 1.\\                          \hspace{10mm}ELSE, return $\mathsf{K}^*$. \\
 \hspace{3mm}ELSE, set and $flag = 1$, $\mathsf{\hat{K}}=\mathsf{K}^*$ and goto step 1.
\end{algorithmic}
}
\end{algorithm}

We show in the experiment section, that the Algo.~\ref{alg:calibration} converges in very few iterations.
In every iteration, beside the reconstruction itself, the major computation is only required while minimizing  $\Phi(\mathsf{K})$. Recall that, $\Phi(\mathsf{K})$ is minimized iteratively using a local method. During local search, the reconstruction for every update is required to compute $\Phi(\mathsf{K})$. Thanks to the upgrade equation, the cost
$\Phi(\mathsf{K})$ can be computed instantly, without going through the computationally expensive reconstruction process. 

\subsection{Intrinsics Recovery in Practice}
Although our reconstruction method makes inextensible shape assumption, the upgrade strategies use the piece-wise rigidity constraint. Despite the fact that the piece-wise rigid assumption is mostly true for inextensible shapes, it could be problematic in certain cases, for example, when the reconstructed points are too sparse. Therefore, some special care need to taken for a robust calibration.    

\noindent{\bf Distance normalization and geodesics:} 
Recall that the upgrade equation~\eqref{eq:upgradeEq} is an approximation under the assumption that either the neighboring image points are sufficiently close to each other or a good guess $\hat{\mathsf{K}}$ is provided. When neither of these conditions are satisfied, the intrinsics obtained from energy minimization may not be sufficiently accurate. While a larger focal length may reduce the residual error, it also reduces individual distances creating  disparities in the reconstruction scale of different views. Therefore, during each iteration of refinement, we fix the scale by enforcing, 
\begin{equation}
\sum_{i}\sum_{j\in\mathcal{N}(i)}\hat{d}_{ij}^l(\mathsf{K}) =1, \forall l. 
\end{equation}
Another important practical aspect here is the use of geodesics $\hat{g}^l(i,j)$ instead of $\hat{d}^l_{ij}$ in Eq.~\eqref{eq:Phi} or Eq.~\eqref{eq:Phi_T}. When the scene points are sparse, using geodesics instead of the local euclidean distances may be necessary. We therefore choose to use geodesics computed from Dijkstra's algorithm~\cite{Dijkstra1959} instead of the local euclidean distances for stability.

\noindent{\bf Re-reconstruction and re-calibration:}  For a better calibration accuracy, especially when the initial guess $\mathsf{\hat{K}}$ is largely inaccurate, we iteratively perform re-reconstruction and re-calibration, starting from newly estimated intrinsics, until convergence. This has already been included in Algo.~\ref{alg:calibration}, which we also included on top of Algo.~\ref{alg:templateCalibration} in our implementation. In practice, only a few such iterations are sufficient to converge, even when the initial guess on intrinsics is very arbitrary. 

\section{Incremental Semi-dense NRSfM}
The SOCP problem of~\eqref{eq:withBoundNRSfM} has the time complexity of $O(n^3)$.
Therefore in practice, only  a sparse set of points can be reconstructed in this manner.
Here, we present a method to iteratively densify the initial sparse reconstruction, followed by online new view/camera addition strategy.
Besides many obvious importance of incremental reconstruction, it is also necessary in our context:  
(a) to allow the selection of the closest image point pairs for camera calibration, (b) to compute 3D Geodesic distances for single view reconstruction.

\subsection{Adding New Points}
Let  $\mathcal{P}$ represents a set of sparse points reconstructed using~\eqref{eq:withBoundNRSfM}.  
We  would like to reconstruct a set of new points  $\mathcal{Q}$ with depths $\zeta^l_i$, such that $\mathcal{Q}\cap\mathcal{P}=\emptyset$, consistent to the
existing reconstruction. This can be achieved by solving the following convex optimization problem,    
{\small
\begin{equation}
\begin{aligned}
& \underset{\zeta_{i}^l, e_{ij}, \alpha}{\textnormal{max}}
& & \alpha\Lambda+\sum_{l} \sum_{i\in\mathcal{Q}} \zeta_{i}^l, \\
& \textnormal{~s.t.}
& & \norm{\mathsf{\hat{K}}^{-1}(\zeta_{i}^l\mathsf{u}_{i}^l - \alpha\lambda_{j}^l\mathsf{u}_{j}^l)}\leq  e_{ij},& j\in \mathcal{N}_p(i),\\
& & & \norm{\mathsf{\hat{K}}^{-1}(\zeta_{i}^l\mathsf{u}_{i}^l - \zeta_{j}^l\mathsf{u}_{j}^l)}\leq  e_{ij},& j\in \mathcal{N}_q(i),\\
&&&\sum_{i}\sum_{j\in\mathcal{N}_q(i)}e_{ij}=1-\alpha,
\end{aligned}
\label{eq:addingPoints}
\end{equation}
}
where, $\Lambda=\sum_{l} \sum_{i\in\mathcal{P}} \lambda_{i}^l$, $\mathcal{N}_p(i)=\mathcal{N}(i)\cap\mathcal{P}$, and ${\mathcal{N}_q(i)=\mathcal{N}(i)\cap\mathcal{Q}}$. The scalars $\alpha$ and $1-\alpha$ represent the contributions of initial reconstruction $\mathcal{P}$ and new reconstruction $\mathcal{Q}$, respectively. Note that the newly reconstructed points respect the inextensible criteria not only among themselves but also with respect to the initial reconstruction. This maintains the consistency between reconstructions $\mathcal{P}$ and $\mathcal{Q}$. The incremental dense reconstruction process iteratively adds disjoint sets $\mathcal{Q}_1,\mathcal{Q}_2,\ldots\mathcal{Q}_r$ to the initial reconstruction $\mathcal{P}$, where $\mathcal{P}$ encodes the overall shape and $\mathcal{Q}_r$ represents the details.

\subsection{Adding New Cameras} 
Adding a new camera to the NRSfM reconstruction is fundamentally a template-based reconstruction problem. If the camera is calibrated, one can obtain the reconstruction directly from~\eqref{eq:basicNRSfM}. For the uncalibrated case, the camera can be calibrated first using~\eqref{eq:templateDIAC}, and the reconstruction upgraded  from~\eqref{eq:withBoundNRSfM} using~\eqref{eq:upgradeEq}. It is important to note that the computation of accurate template geodesic distances $d_{ij}$, as required for template-based reconstruction, is possible only when the reconstruction is dense enough. This is not really a problem, thanks to the proposed incremental reconstruction method.

\section{Discussion}
\noindent{\bf Initial guess $\mathsf{\hat{K}}$:} In all our experiments, we choose the initial guess $\mathsf{\hat{K}}$ by setting both focal lengths to the half of the mean image size and principal point to the image center.

\noindent{\bf Missing features:} Feature points may be missing from some images due to occlusion or matching failure. This problem can be addressed during reconstruction by discarding all the variables corresponding to missing points together with all the inextensible constraints involving them as done in \cite{Chhatkuli2016,Vicente2012}.

\noindent{\bf Reconstruction Consistency:}
Alternative to~\eqref{eq:addingPoints}, one can also think of reconstructing two overlapping sets $\mathcal{P}$ and $\mathcal{Q}$ such that 
$\mathcal{P}\cap\mathcal{Q}=\mathcal{R}$ independently. Then, the registration between them can be done with the help of $\mathcal{R}$ from two sides. However, this is not only computationally inefficient due to the overlap, but also geometrically inconsistent.

\begin{figure*}[tb]
\centering
\small
  \begin{minipage}{1\textwidth}
   \includegraphics[width=0.24\textwidth]{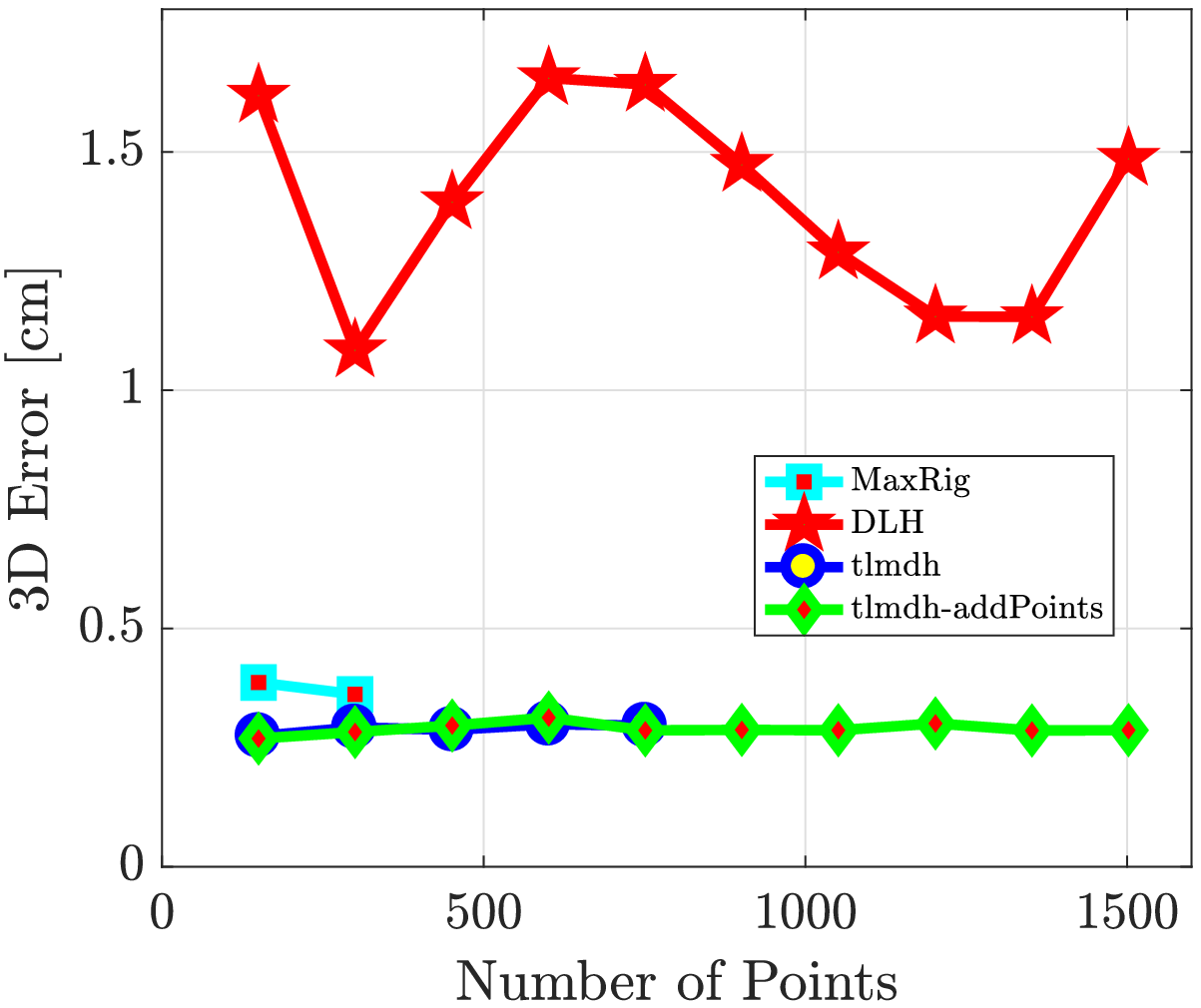}
   \includegraphics[width=0.24\textwidth]{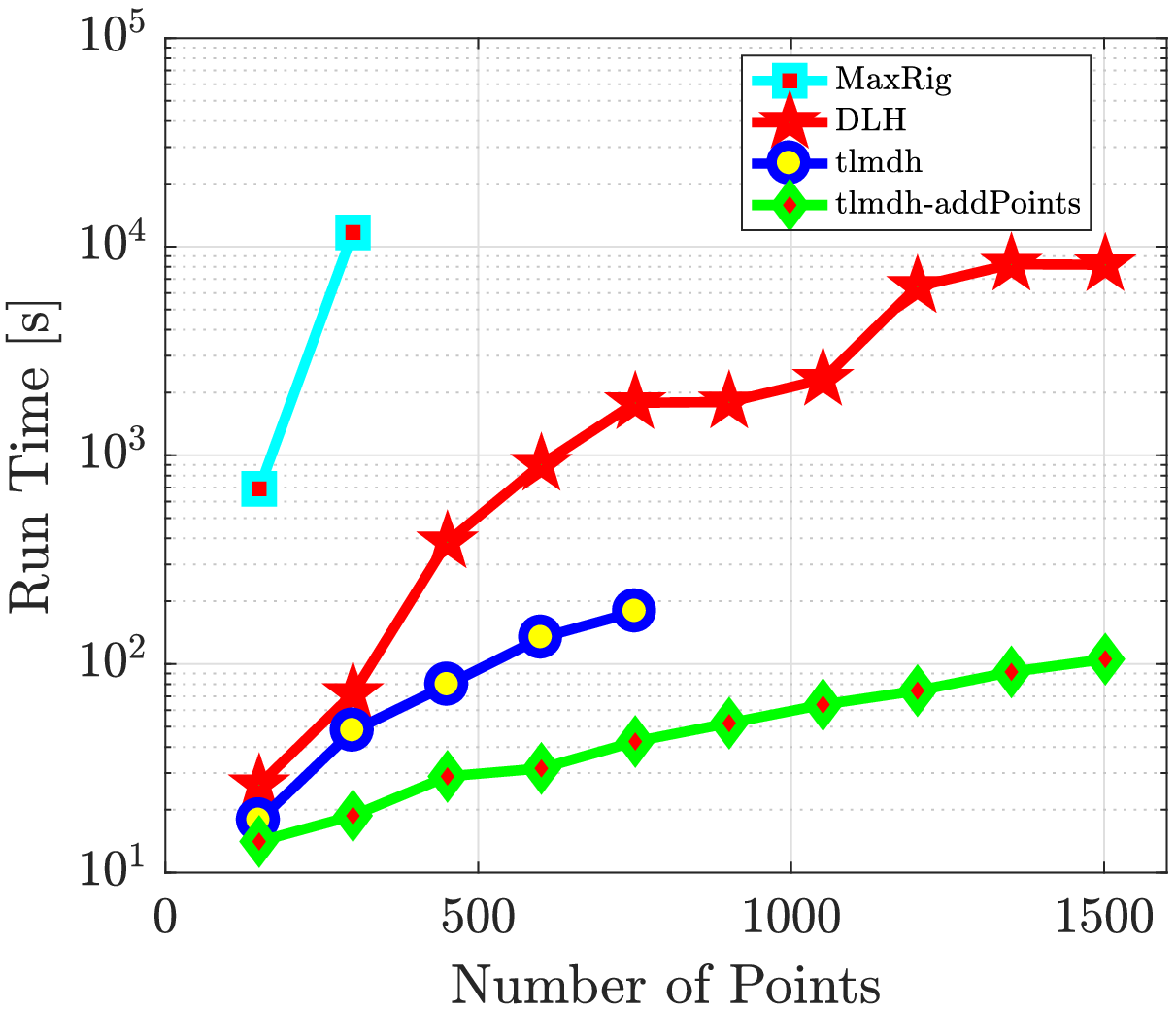}
   \includegraphics[width=0.24\textwidth]{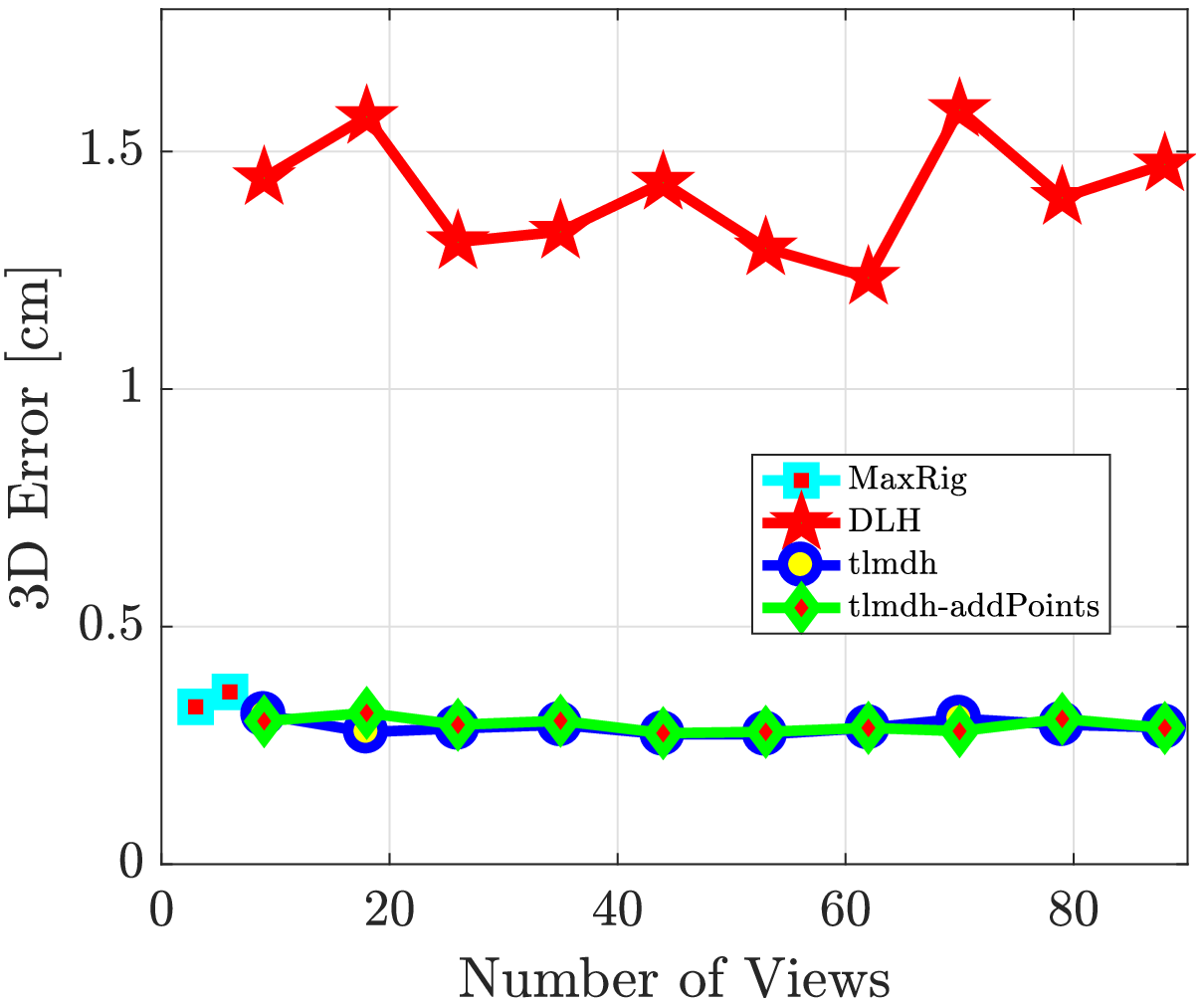}
   \includegraphics[width=0.24\textwidth]{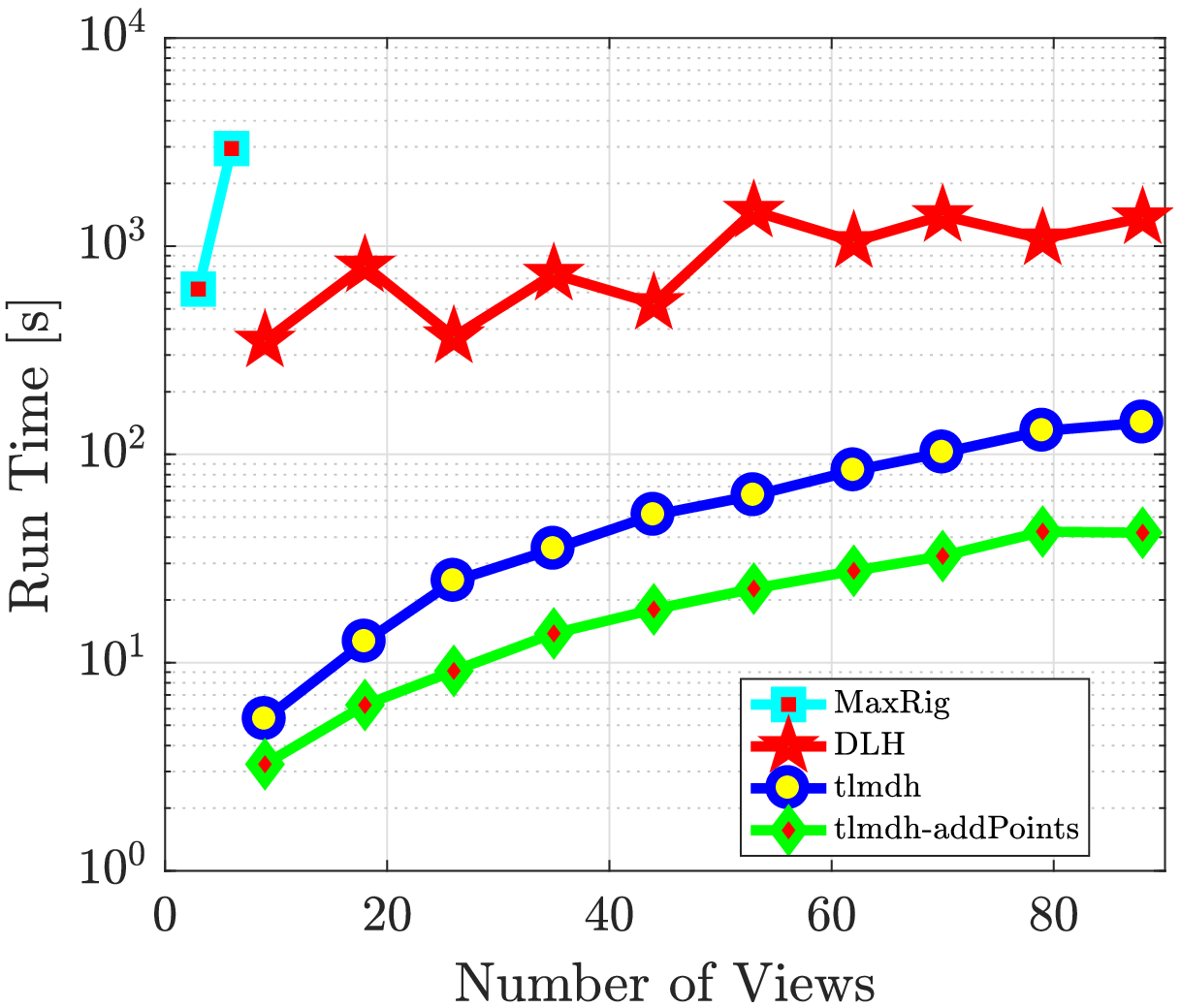}
  \end{minipage} 
(a) Adding points: 25\% of points are added incrementally to the initial reconstruction.
  \begin{minipage}{1\textwidth}
   \includegraphics[width=0.24\textwidth]{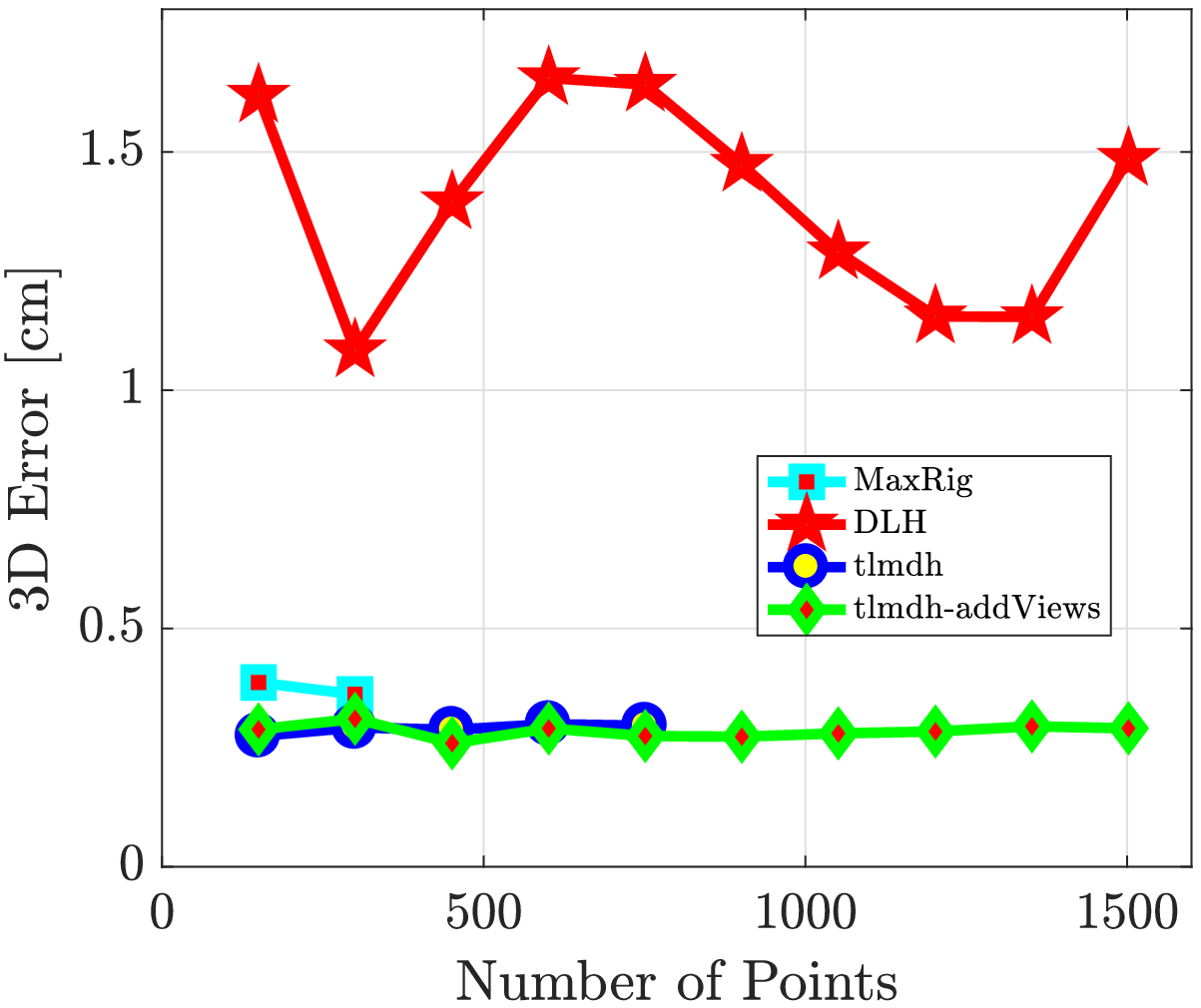}
   \includegraphics[width=0.24\textwidth]{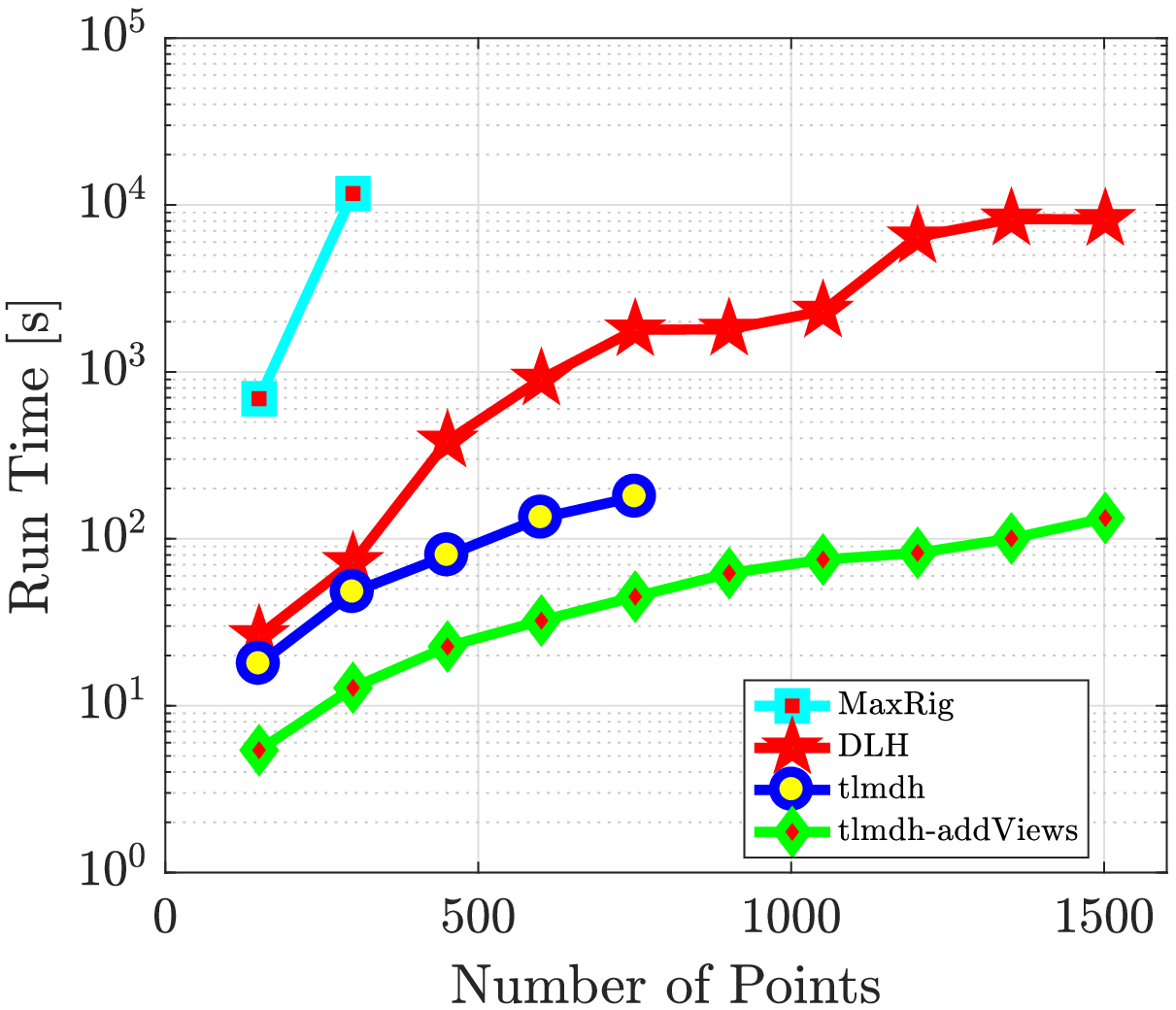}
   \includegraphics[width=0.24\textwidth]{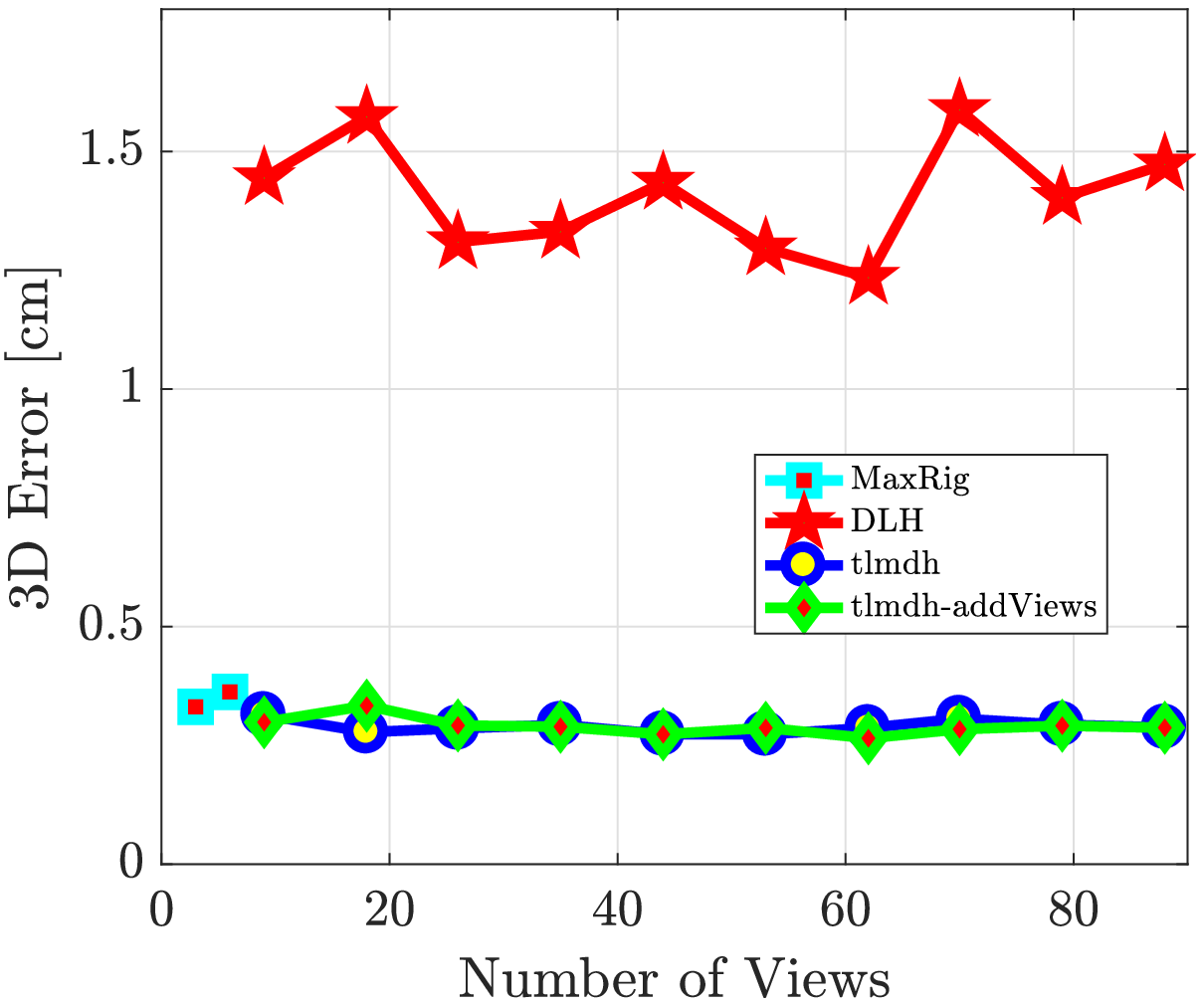}
   \includegraphics[width=0.24\textwidth]{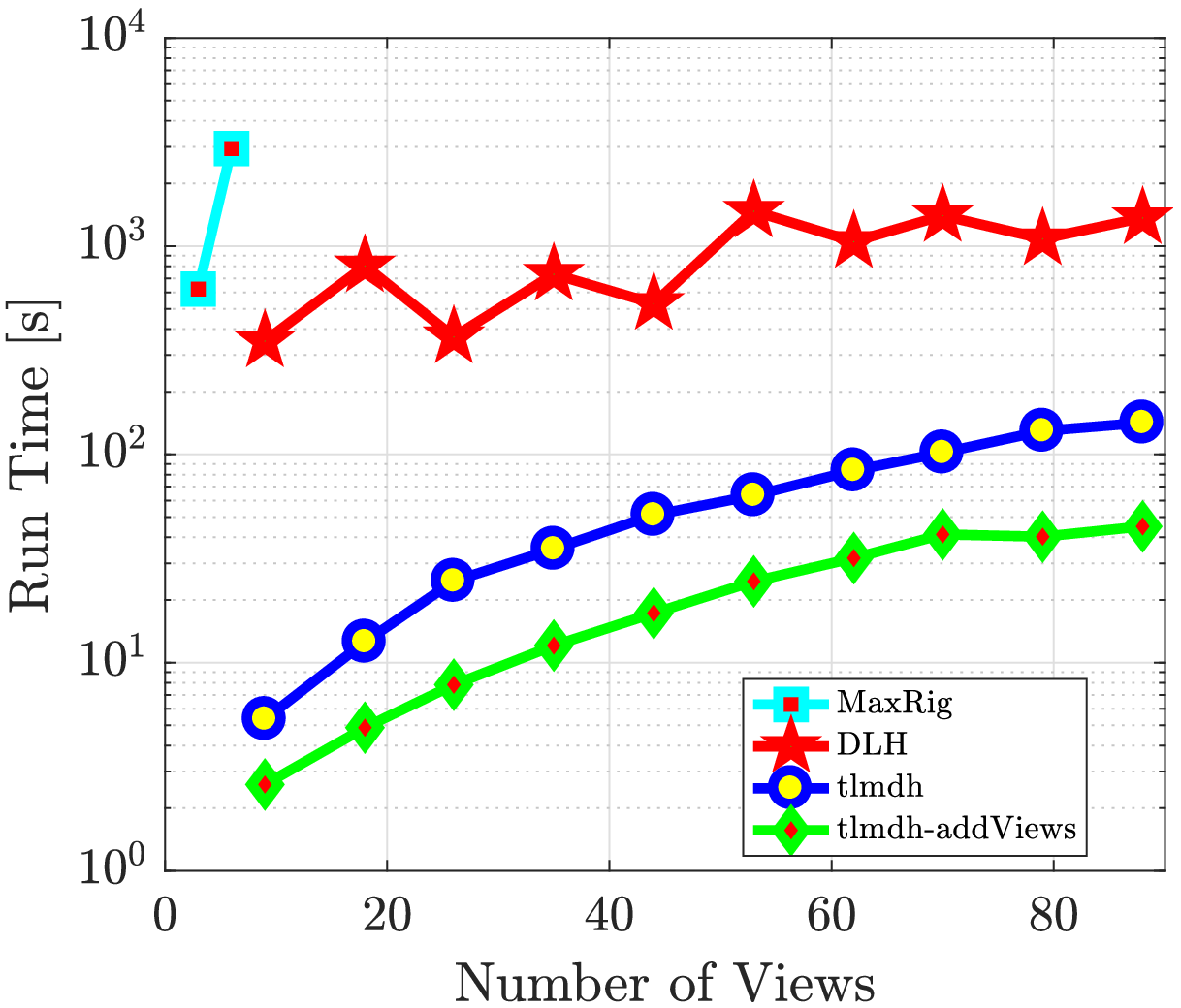}
  \end{minipage}
(b) Adding cameras: Half of the views are added to the initial reconstruction. 
\caption{\textbf{Incremental Semi-dense NRSfM.} Comparison of reconstruction error and run time on the \emph{Hand} dataset. Left: varying number of points (number of views=88). Right:  varying number of views (number of points=751). Run time shown in log scale.} 
\label{fig:incrementalRec}
\end{figure*}

\section{Experimental Results}
\label{sec:results}
We conduct extensive experiments in order to validate the presented theory and to evaluate the performance, run time and practicality of the proposed methods.  

\paragraph{Datasets.}
We first provide a brief descriptions of the datasets we use to analyze our algorithms.
\noindent{\bf KINECT Paper.} This VGA resolution image sequence shows a textured paper deforming smoothly~\cite{kinectpaper}. The tracks contain about 1500 semi-dense but noisy points.
\noindent{\bf Hulk \& T-Shirt.} The datasets contain a comic book cover in 21 different deformations, and a textured T-Shirt with 10 different deformations~\cite{ChhatkuliBMVC2014}, in high resolution images.
Although the number of points is low (122 and 85, resp.), the tracks have very little noise and therefore we obtain a very accurate auto-calibration.
\noindent{\bf Flag.} This semi-synthetic dataset is created from mocap recordings of deforming cloth~\cite{White2007}. We generate 250 points in 30 views using a virtual 640x480 perspective camera.
\noindent{\bf Newspaper.} This sequence\footnote{\label{fn:data}The dataset was provided by the authors.} contains the deformation and tearing of a double-page newspaper, recorded with KINECT in HD resolution~\cite{Chhatkuli2016}. 
\noindent{\bf Hand.} The Hand dataset~\cite{Chhatkuli2016} features medium resolution images. Dense tracking~\cite{Sundaram2010} of image points yield up to 1500 tracks in 88 views. The dataset consists of ground-truth 3D for the first and the last image of the sequence.
\noindent{\bf Minion \& Sunflower.} These sequences are recorded with a static Kinect sensor~\cite{Innmann2016}. Minion contains a stuffed animal undergoing folding and squeezing deformations. Sunflower however features only small translation w.r.t. the camera. We incrementally reconstruct more than 10,000 points for Minion, and 5,000 for Sunflower, as shown in Fig.~\ref{fig:denseQualitativeResultsFrontPage}. We are able to reconstruct the global deformation, and mid-level details such as the glasses of Minion. Unfortunately, due to the failure of optical flow tracking, we fail to reconstruct homogeneous areas and fine details. In Sunflower we can capture the deformation of the outside leafs, whereas finer details in the center of the blossom is not recovered due to insufficient change in viewpoint.
\noindent{\bf Camel\footnote{https://www.youtube.com/watch?v=PhpeadpZsa4} \& Kitten\footnote{https://www.youtube.com/watch?v=DIZM2OMNc7c}.} We took two sequences from YouTube videos to show the incremental semi-dense NRSfM from uncalibrated cameras. The camel turns around its head towards the moving camera, providing enough motion to faithfully reconstruct the 3D motion of the animal. Fig.~\ref{fig:denseQualitativeResultsFrontPage} shows the 3D structure of  more than 3,000 points for one out of 61 views reconstructed. In the Kitten sequence (18,000 points for each 36 views), a cat performs both articulated and deforming motion with body and tail. Again, state-of-the-art optical flow methods struggle to maintain stable points tracks, especially on the head. Nevertheless, our method captures the general motion to a very good extent. In all of the above datasets, \textbf{DLH} fails to get the correct shape while \textbf{MaxRig} cannot reconstruct the shape faithfully as it cannot handle enough points.
\noindent{\bf Cap.} This dataset contains wide-baseline views of a cap in two different deformations~\cite{Bartoli2013}. The 3D template of the undeformed cap was obtained using SfM pipeline for the images from the first camera. Then, the second camera is calibrated using our template-based method.

\subsection{Camera Calibration from a Non-rigid Scene}
To measure the quality of our calibration results, we report the 3D root mean square error (RMSE), the relative focal length and principal point estimation error. Furthermore, we provide the number of iterations and the corresponding run times in Table~\ref{tab:resultsFocalLengthEst}. 

\begin{table}
\centering
\setlength\tabcolsep{1.8pt}
\scriptsize
\begin{tabular}{l|c c|c c c|c c | c c|c c|c c}
\multirow{ 2}{*}{Dataset} & \multicolumn{2}{c|}{Number of} & \multicolumn{3}{c|}{Run time [s]} & \multicolumn{4}{c|}{Focal Estimation} & \multicolumn{4}{c}{Reconstruction Error $E_{\text{rec}}$}\\
& Points & Views &  $T_{\text{orig}}$ & $N_{\text{iter}}$ & $T_{\text{total}}$ & $f_{\text{init}}$& $f_{\text{GT}}$ & $f_{\text{est}}$ & Error \% & \multicolumn{2}{c|}{$f_{\text{GT}}$} & \multicolumn{2}{c}{$f_{\text{est}}$}\\
\hline
\multicolumn{14}{c}{Template-based focal length estimation}\\ \hline
KINECT Paper	& 301 	& 23 	& 2.3 	& - 	& 16.8 	& - 	&  528	& 590	& 11.74 & 3.00 		& 0.54\% & \textbf{2.83} & 0.50\%\\
Hulk 		& 122 	& 21 	& 0.4 	& - 	& 4.2 	& - 	&  3784	& 4300	& 13.61 & 5.73 		& 1.43\% & \textbf{5.53} & 1.37\%\\
Flag 		& 250	& 30 	& 1.3 	& - 	& 178.2 & - 	&  384 	& 420	& 9.38 	& 4.74 		& 0.58\% & \textbf{4.54} & 0.56\%\\
Cap		& 137	& 1 	& 0.3 	& - 	& 11.0 	& - 	&  2039	& 2300	& 12.8	& \textbf{1.13} & 4.80\% & \textbf{1.13} & 4.80\%\\ \hline
\multicolumn{14}{c}{Template-less focal length estimation} \\\hline
KPaper 		& 301 	& 23 	& 5.8 	& 3 	& 110.1	& 280 	& 528 	& 540	& 2.27 & 4.44 		& 0.80\%  & \textbf{4.28} & 0.77\%\\
Hulk 		& 122 	& 21 	& 1.9 	& 5 	& 36.5 	& 1641 	& 3784 	& 3800 	& 0.40 & 2.76 		& 0.67\%  & \textbf{2.75} & 0.66\%\\
T-Shirt 	& 85 	& 10 	& 0.6 	& 10 	& 24.1 	& 2000 	& 3787 	& 4000	& 5.63 & 3.52 		& 1.10\%   & \textbf{3.42} & 1.07\% \\
Flag		& 250 	& 30 	& 2.6 	& 6 	& 185.4	& 280 	& 384 	& 400 	& 4.17 & 5.24 		& 0.64\% & \textbf{5.05} & 0.62\%\\
Newspaper 	& 441 	& 19 	& 24.5 	& 5 	& 523.6 & 750 	& 1055 	& 870	& 16.6 & \textbf{7.79} 	& 1.09\% & 9.27 & 1.30\%\\%
\end{tabular}
\caption{\textbf{Focal Length Estimation from a Non-Rigid Scene.} We report the run-time, reconstruction error and relative focal length estimation error of our template-based and template-less NRSfM calibration methods. $T_{\text{orig}}$ is the time needed to reconstruct with a given focal length, $T_{\text{total}}$ the run time including calibration. For the template-less case, $N_{\text{iter}}$ iterations were performed until convergence.}
\label{tab:resultsFocalLengthEst}
\end{table}

\noindent{\bf Template-based Camera Calibration.}
In the first part of Algo.~\ref{alg:templateCalibration} we generate hypotheses for $\mathsf{K}$ and choose the one with best isometric match with the template. We perform experiments on the KINECT Paper, Hulk and Flag dataset and report the results in Table~\ref{tab:resultsFocalLengthEst}. We observe a consistent improvement in reconstruction accuracy with the estimated intrinsics.
The second part of Algo.~\ref{alg:templateCalibration} involves gradient-based refinement on the intrinsics by minimizing Eq.~\eqref{eq:templateCalibCost}. To analyze this part, we conduct two experiments: First, we perform refinement on the initially estimated intrinsics $f_{\text{poly}}$. Here we can consistently improve reconstruction errors with the refined intrinsics.
In the Hulk and Flag dataset, we also get a better estimate of the focal length. On KINECT Paper however, the focal length deteriorates, while reconstruction accuracy improves. This is most probably due to the noisy tracks in the sequence. Due to the effective regularization, the error in principal point is consistently low. In the second experiment, we gauge the robustness of our refinement method. To this end, we simulate initial intrinsics by adding $\pm 20\%$ uniform noise independently on each of the entries of $\mathsf{K}_{GT}$, and compare reconstruction error and the refined intrinsics shown in Table~\ref{table:calibrationRefinement}. We compare to Bartoli~et~al.~\cite{Bartoli2013} on the Cap dataset directly from the paper, since it is non-trivial to implement the method itself. We observed an error $E_{\text{f}}$ of about 13\% with our method, compared to 3.8\%-7.3\% reported by~\cite{Bartoli2013}. The slightly higher error in the Cap dataset can be partly attributed to the repeating texture that makes our image matches non-ideal. Overall we can observe a consistent improvement in almost all metrics, validating the robustness of the method and the assumptions it is based on.

\begin{table}
\centering
\newcommand{\cWidth}{0.6cm}
\newcommand{\hcWidth}{0.7cm}
\setlength\tabcolsep{0.7pt}
\scriptsize
 \begin{tabular}{l | C{\cWidth} || C{\cWidth}:C{\cWidth}:C{\cWidth} | C{\cWidth}:C{\cWidth}:C{\cWidth}:C{\cWidth} || C{\cWidth}C{\hcWidth}:C{\cWidth}C{\hcWidth}:C{\cWidth}C{\hcWidth}}
 \multirow{ 2}{*}{Dataset} & \multicolumn{1}{c||}{} & \multicolumn{3}{c|}{Template-based} & \multicolumn{4}{c||}{Refined} & \multicolumn{6}{c}{Simulated initial $\mathsf{K}$ ($10^3$ samples avg.)}\\  
& $f_{\text{GT}}$ & $f_{\text{poly}}$ & $E_{\text{f}}$   & $E_{\text{rec}}$ & $f_{\text{ref}}$ & $E_{\text{f}}$ & $E_{\text{PP}}$ & $E_{\text{rec}}$ &$E_{\text{f}}$ & $\Delta E_{\text{f}}$ & $E_{\text{PP}}$ & $\Delta E_{\text{PP}}$ & $E_{\text{rec}}$ & $\Delta E_{\text{rec}}$ \\ \hline
KINECT Paper 	& 528	& 590	& \textbf{11.74}& 2.83 	& 604	 & 14.45 	& 0.04 	&\textbf{2.73}	& 8.87 	& \textbf{-0.37}& 0.05	& \textbf{-10.96} & 3.82 & \textbf{-0.25} \\
Hulk  		& 3784	& 4300	& 13.61 & 5.53 		& 4119	 & \textbf{8.85}&  1.74 &\textbf{5.53} 	& 7.30 	& \textbf{-2.36}&  1.77 & \textbf{-8.84}  & 6.52 & \textbf{-0.01} \\
Flag 		& 384	& 420	& 9.38  & 4.54		& 414	 & \textbf{7.98}&  0.05 &\textbf{4.34} 	& 8.61 	& \textbf{-1.05}& 0.08 	& \textbf{-10.45} & 6.05 & +0.08\\
Cap 		& 2039	& 2300	& \textbf{12.8}	& \textbf{1.13}	& 2360	 & 13.1&  2.33	&\textbf{1.13} 	& 9.18 	& \textbf{-0.10}& 2.33 	& \textbf{-8.42} & 1.48 & \textbf{-0.00}\\
 \end{tabular}
 \caption{\textbf{Calibration Refinement.} We compute the full calibration $\mathsf{K}$ by initializing with the template-based calibration $f_{\text{poly}}$, and test the robustness by adding synthetic noise on the $\mathsf{K_{\text{GT}}}$. Reconstruction errors $E_{rec}$ are in mm, others in \%.}
 \label{table:calibrationRefinement}
%
%
\end{table}

\noindent{\bf Template-less Camera Calibration.}
To visualize the dynamics of Algo.~\ref{alg:calibration}, we plot the error in isometry $\Phi(\mathsf{K})$ over focal length for each iteration on the Hulk dataset in Fig.~\ref{fig:focalLengthEst} (a). Typically, less than 10 iterations are necessary for the method to converge. As we hypothesized above, Fig.~\ref{fig:focalLengthEst} (b) empirically verifies that we can find the termination criterion for our sweeping strategy by thresholding the focal length change $\delta(\mathsf{K}^*,\mathsf{\hat{K}})$. Our method consistently recovers a correct estimate of the intrinsics as reported in Table~\ref{tab:resultsFocalLengthEst}. Moreover, the fact that we obtain better reconstruction accuracy in almost all datasets validates our approach of using the isometric consistency $\Phi(\mathsf{K})$.

\begin{figure}
\centering
\scriptsize
  \begin{minipage}[c]{0.65\textwidth}
  \centering
   \includegraphics[width=0.49\textwidth]{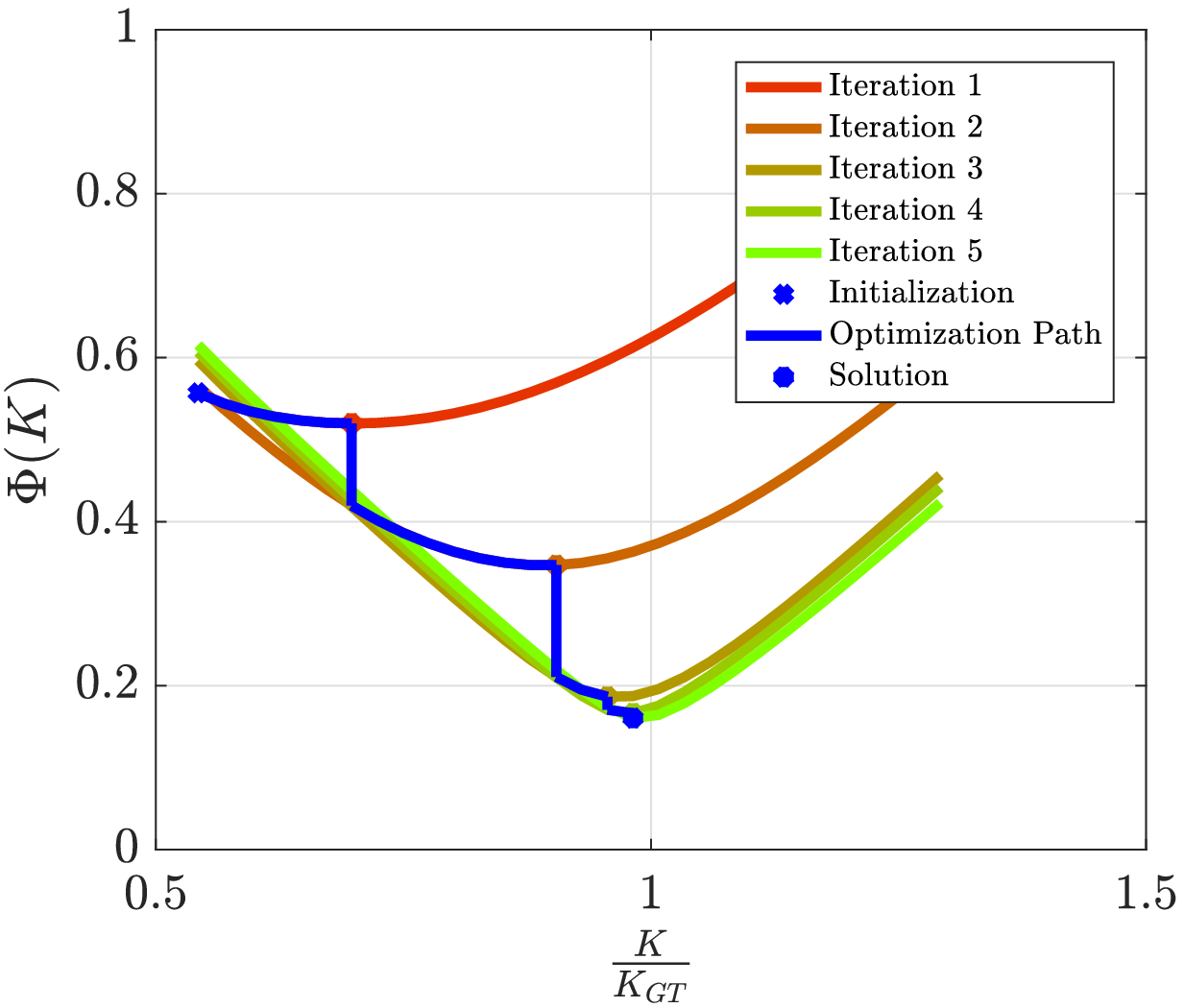}
   \includegraphics[width=0.49\textwidth]{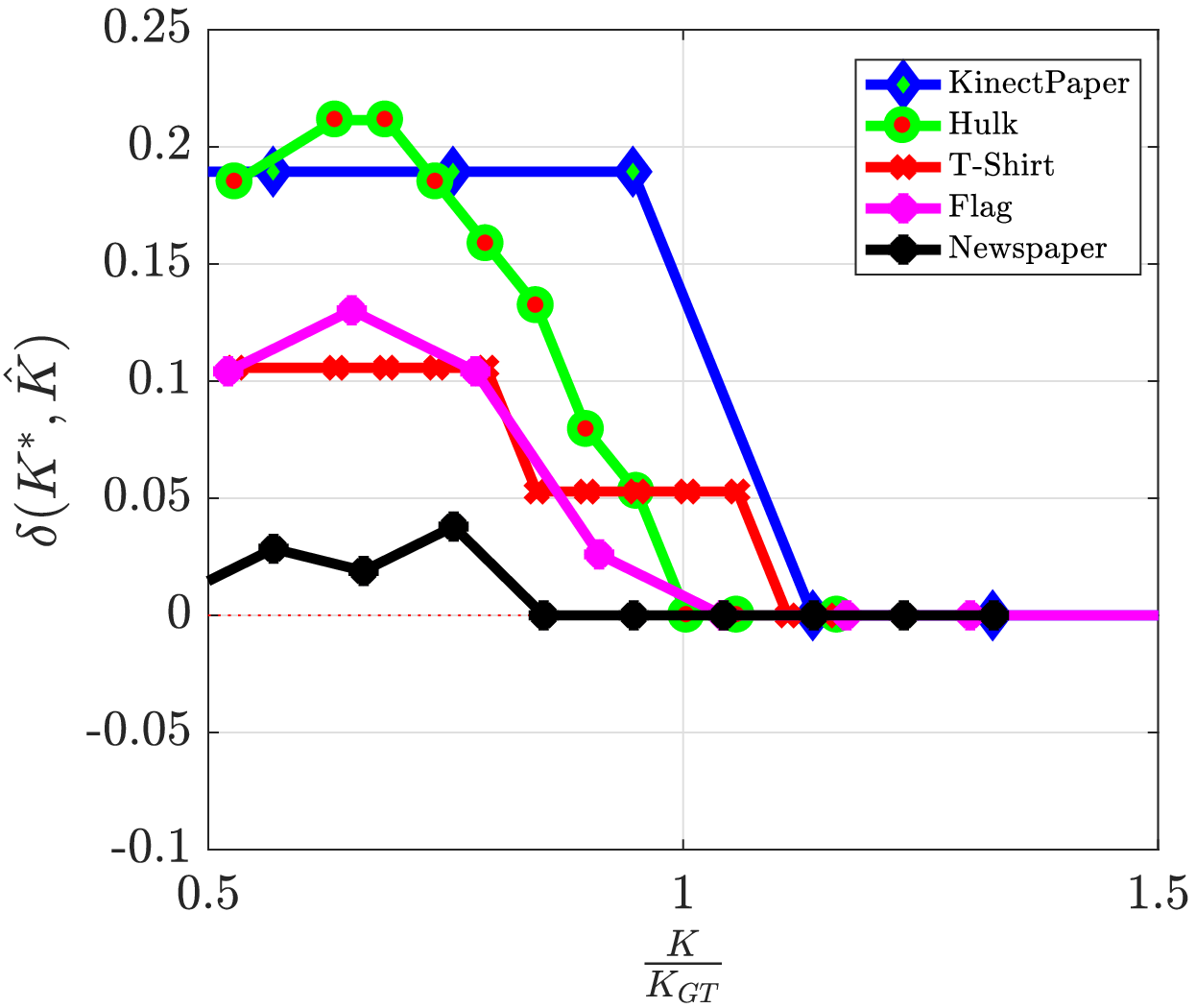}
   \end{minipage}\hfill
   \begin{minipage}[c]{0.34\textwidth}
   (a) Left: in each iteration of step 2, we look for a $\mathsf{K}^*$ that minimizes the error in isometry $\Phi(\mathsf{K})$.\\ \\ \\
   (b) Right: in step 3 we query the focal length gap $\delta(\mathsf{K}^*,\mathsf{\hat{K}})$, and terminate when it becomes sufficiently small.
  \end{minipage}
\caption{\textbf{Template-less Calibration (Algo.~\ref{alg:calibration}).} We iteratively search the smallest $\mathsf{K}$ that maximizes isometry.} 
\label{fig:focalLengthEst}
\end{figure}

\begin{table}
\centering
\scriptsize
\setlength\extrarowheight{2pt}
\setlength\tabcolsep{0.5pt}
\scalebox{0.9}{
\newcommand{\cWidth}{1cm}
\begin{tabular}{l C{\cWidth} C{\cWidth} : C{\cWidth} C{\cWidth} : C{\cWidth} : C{\cWidth} : C{\cWidth} : C{\cWidth} : C{1.3cm} }
\hline
\multicolumn{1}{l}{Datasets} & \multicolumn{2}{c}{\textbf{incr-tlmdh}} & \multicolumn{2}{c}{\textbf{tlmdh}} & \textbf{p-isomet} & \textbf{p-isolh} & \textbf{DLH} & \textbf{o-kfac} \\ \midrule
\multicolumn{1}{l}{KPaper} & \textbf{4.64} & \textbf{176.16s} & 5.41 & 605.06s & 7.63 & 13.64 & 14.66 & 13.93  \\
\multicolumn{1}{l}{Hulk} & 2.99 & \textbf{0.80s} & \textbf{2.76} & 1.99s & 10.76 & 14.54 & 22.98 & -   \\
\multicolumn{1}{l}{T-Shirt} &3.83 & \textbf{0.23s} & \textbf{3.53} & 0.47s & 10.60 &  8.94 & - & -  \\
\multicolumn{1}{l}{Cardboard} &\textbf{13.22} & \textbf{18.94s} & 14.56 & 34.35s & - & 12.95 & - & -  \\
\multicolumn{1}{l}{Rug} &\textbf{26.40} & \textbf{205.89s} & 26.60 & 542.39s & 26.15 & 38.26 & 31.01 & - \\
\multicolumn{1}{l}{Table mat} &15.99 & \textbf{5.54s}  & 14.36 & 7.65s & \textbf{14.21} & 20.71 & 17.51 & 16.24  \\
\multicolumn{1}{l}{Newspaper} & \textbf{10.79} & \textbf{89.27s} & 11.63 & 190.96s & 18.40 & 37.21 & 24.94 & 30.74  \\ \bottomrule
\end{tabular}
}
\caption{\textbf{Comparison of NRSfM methods.} Mean 3D errors in mm and run time comparison for batch and incremental reconstruction in real datasets.}
\label{tab:resultsCompareAll}
\end{table}

\subsection{Incremental Reconstruction}

We first present experiments on the dense Hand dataset in Fig.~\ref{fig:incrementalRec}. We compare to two state-of-the-art NRSfM approaches, \textbf{MaxRig}~\cite{Ji2017} and \textbf{DLH}~\cite{Dai2012}, as well as the to batch version of our approach \textbf{tlmdh}~\cite{Chhatkuli2016}. In the first row, we plot the performance of \textbf{tlmdh-addPoints}: we start by reconstructing a random subset of $\max \{150, \frac{N}{4×}\}$ points, and incrementally add the remaining points in subsequent iterations according to Eq.~\eqref{eq:addingPoints}. While achieving competitive reconstruction accuracy on par with \textbf{tlmdh}, we observe remarkable advantages in run time compared to all other methods. \textbf{MaxRig} shows good accuracy, but suffers from serious run time and memory problems. \textbf{DLH} on the other hand is slow and exhibits poor accuracy on this dataset, due to perspective and non-linear deformations. The second row of Fig.~\ref{fig:incrementalRec} shows the same experimental setup with \textbf{tlmdh-addViews}. Here, we reconstruct all points at once, but incrementally add the remaining 50\% of views to the reconstruction of the first half. To this end, we compute the template from the first reconstruction and employ SfT. The graphs clearly show that \textbf{tlmdh-addViews} exhibits a favorable run time complexity without impairing the reconstruction accuracy. We provide more results in the supplementary material.
Furthermore, we perform extensive experiments on a variety of additional datasets, and compare with the reconstructions of \textbf{p-isomet}~\cite{Parashar2016}, \textbf{p-isolh}~\cite{ChhatkuliBMVC2014}, \textbf{DLH}~\cite{Dai2012}, and \textbf{o-kfc}~\cite{Gotardo2011} in Table~\ref{tab:resultsCompareAll} obtained from \cite{Chhatkuli2017}. Overall, we observe a significant advantage in accuracy and run time in particular compared to the best performing baseline \textbf{tlmdh}.

\section{Conclusions}
\label{sec:conclusion}
In this paper we formulated a method addressing the unknown focal-length in NRSfM and unknown intrinsics in SfT. Despite the computational complexity of convex NRSfM, we formulated an incremental framework to obtain semi-dense reconstruction and reconstruct new views. We developed our theory based on the surface isometry prior in the context of the perspective camera. We developed and verified our approach for intrinsics/focal-length recovery for both template-based and template-less non-rigid reconstruction. Essential to our method is a novel upgrade equation, that analytically relates reconstructions for different intrinsics. We performed extensive quantitative and qualitative analysis of our methods on different datasets which shows the proposed methods perform well despite addressing very challenging problems.

\section*{Acknowledgements}
Research was funded by the EU's Horizon 2020 programme under  grant No.\ 645331-- EurEyeCase and grant No.\ 687757-- REPLICATE,  and the Swiss Commission for Technology and Innovation (CTI, Grant No.  26253.1 PFES-ES, EXASOLVED).

{\small
\bibliographystyle{splncs}
\bibliography{egbib}

\begin{thebibliography}{10}

\bibitem{Longuet1981}
Longuet-Higgins, H.:
\newblock A computer algorithm for reconstructing a scene from two projections.
\newblock Nature \textbf{293} (1981)  133--135

\bibitem{Nister2004}
Nist{\'e}r, D.:
\newblock An efficient solution to the five-point relative pose problem.
\newblock IEEE Trans. Pattern Anal. Mach. Intell. \textbf{26}(6) (2004)
  756--777

\bibitem{Hartley2004}
Hartley, R.I., Zisserman, A.:
\newblock Multiple View Geometry in Computer Vision. Second edn.
\newblock Cambridge University Press, ISBN: 0521540518 (2004)

\bibitem{Photoscan}
PhotoScan, A.:
\newblock Agisoft PhotoScan User Manual Professional Edition, Version 1.2.
  (2017)

\bibitem{Autodesk}
ReCap, A.:
\newblock ReCap 360 -- Advanced Workflows. (2015)

\bibitem{Ji2017}
Ji, P., Li, H., Dai, Y., Reid, I.:
\newblock {"Maximizing Rigidity"} revisited: A convex programming approach for
  generic 3d shape reconstruction from multiple perspective views.
\newblock In: ICCV. (2017)

\bibitem{Dai2012}
Dai, Y., Li, H., He, M.:
\newblock A simple prior-free method for non-rigid structure-from-motion
  factorization.
\newblock In: CVPR. (2012)

\bibitem{Bregler2000}
Bregler, C., Hertzmann, A., Biermann, H.:
\newblock Recovering non-rigid 3{D} shape from image streams.
\newblock In: CVPR. (2000)

\bibitem{Torresani2008}
Torresani, L., Hertzmann, A., Bregler, C.:
\newblock Nonrigid structure-from-motion: Estimating shape and motion with
  hierarchical priors.
\newblock IEEE Trans. Pattern Anal. Mach. Intell. \textbf{30}(5) (2008)
  878--892

\bibitem{DelBue2008}
Del~Bue, A.:
\newblock A factorization approach to structure from motion with shape priors.
\newblock In: CVPR. (2008)

\bibitem{Garg2013}
Garg, R., Roussos, A., Agapito, L.:
\newblock Dense variational reconstruction of non-rigid surfaces from monocular
  video.
\newblock In: CVPR. (2013)

\bibitem{Fayad2010}
Fayad, J., Agapito, L., Del~Bue, A.:
\newblock Piecewise quadratic reconstruction of non-rigid surfaces from
  monocular sequences.
\newblock In: ECCV. (2010)

\bibitem{Agudo2014bmvc}
Agudo, A., Montiel, J., Agapito, L., Calvo, B.:
\newblock Online dense non-rigid 3d shape and camera motion recovery.
\newblock In: BMVC. (2014)

\bibitem{Taylor2010}
Taylor, J., Jepson, A.D., Kutulakos, K.N.:
\newblock Non-rigid structure from locally-rigid motion.
\newblock In: CVPR. (2010)

\bibitem{Bartoli2013iccv}
Bartoli, A., Pizarro, D., Collins, T.:
\newblock A robust analytical solution to isometric shape-from-template with
  focal length calibration.
\newblock In: ICCV. (2013)

\bibitem{Ngo2016}
Ngo, T.D., \"Ostlund, J.O., Fua, P.:
\newblock Template-based monocular 3{D} shape recovery using laplacian meshes.
\newblock {IEEE} {T}ransactions on {P}attern {A}nalysis and {M}achine
  {I}ntelligence \textbf{38}(1) (2016)  172--187

\bibitem{ChhatkuliPAMI2016}
Chhatkuli, A., Pizarro, D., Bartoli, A., Collins, T.:
\newblock A stable analytical framework for isometric shape-from-template by
  surface integration.
\newblock IEEE Transactions on Pattern Analysis and Machine Intelligence
  \textbf{39}(5) (2017)  833--850

\bibitem{Bartoli2013}
Bartoli, A., Collins, T.:
\newblock Template-based isometric deformable 3{D} reconstruction with
  sampling-based focal length self-calibration.
\newblock In: CVPR. (2013)

\bibitem{Chhatkuli2016}
Chhatkuli, A., Pizarro, D., Collins, T., Bartoli, A.:
\newblock Inextensible non-rigid shape-from-motion by second-order cone
  programming.
\newblock In: CVPR. (2016)

\bibitem{Perriollat2011}
Perriollat, M., Hartley, R., Bartoli, A.:
\newblock Monocular template-based reconstruction of inextensible surfaces.
\newblock International Journal of Computer Vision \textbf{95}(2) (2011)
  124--137

\bibitem{Salzmann2011}
Salzmann, M., Fua, P.:
\newblock Linear local models for monocular reconstruction of deformable
  surfaces.
\newblock IEEE Transactions on Pattern Analysis and Machine Intelligence
  \textbf{33}(5) (2011)  931--944

\bibitem{Kumar2017}
Kumar, S., Dai, Y., Li, H.:
\newblock Monocular dense 3d reconstruction of a complex dynamic scene from two
  perspective frames.
\newblock In: ICCV. (2017)

\bibitem{Russell2014}
Russell, C., Yu, R., Agapito, L.:
\newblock Video pop-up: Monocular 3d reconstruction of dynamic scenes.
\newblock In: ECCV. (2014)

\bibitem{Bartoli2015}
Bartoli, A., G{\'{e}}rard, Y., Chadebecq, F., Collins, T., Pizarro, D.:
\newblock Shape-from-template.
\newblock {IEEE} Trans. Pattern Anal. Mach. Intell. \textbf{37}(10) (2015)
  2099--2118

\bibitem{Vicente2012}
Vicente, S., Agapito, L.:
\newblock Soft inextensibility constraints for template-free non-rigid
  reconstruction.
\newblock In: ECCV. (2012)

\bibitem{Chhatkuli2014}
Chhatkuli, A., Pizarro, D., Bartoli, A.:
\newblock Stable template-based isometric 3{D} reconstruction in all imaging
  conditions by linear least-squares.
\newblock In: CVPR. (2014)

\bibitem{Parashar2016}
Parashar, S., Pizarro, D., Bartoli, A.:
\newblock Isometric non-rigid shape-from-motion in linear time.
\newblock In: CVPR. (2016)

\bibitem{XiaoKanade2005}
Xiao, J., Kanade, T.:
\newblock Uncalibrated perspective reconstruction of deformable structures.
\newblock In: Tenth IEEE International Conference on Computer Vision (ICCV'05)
  Volume 1. Volume~2. (2005)  1075--1082 Vol. 2

\bibitem{Salzmann2007}
Salzmann, M., Hartley, R., Fua, P.:
\newblock Convex optimization for deformable surface 3-{D} tracking.
\newblock In: ICCV. (2007)

\bibitem{Akhter2011}
Akhter, I., Sheikh, Y., Khan, S., Kanade, T.:
\newblock Trajectory space: A dual representation for nonrigid structure from
  motion.
\newblock IEEE TPAMI \textbf{33}(7) (2011)  1442--1456

\bibitem{nister2004untwisting}
Nist{\'e}r, D.:
\newblock Untwisting a projective reconstruction.
\newblock International Journal of Computer Vision \textbf{60}(2) (2004)
  165--183

\bibitem{chandraker2007autocalibration}
Chandraker, M., Agarwal, S., Kahl, F., Nist{\'e}r, D., Kriegman, D.:
\newblock Autocalibration via rank-constrained estimation of the absolute
  quadric.
\newblock In: Computer Vision and Pattern Recognition, 2007. CVPR'07. IEEE
  Conference on, IEEE (2007)  1--8

\bibitem{Dijkstra1959}
Dijkstra, E.W.:
\newblock A note on two problems in connexion with graphs.
\newblock Numer. Math. \textbf{1}(1) (1959)  269--271

\bibitem{kinectpaper}
Varol, A., Salzmann, M., Fua, P., Urtasun, R.:
\newblock A constrained latent variable model.
\newblock In: CVPR. (2012)

\bibitem{ChhatkuliBMVC2014}
Chhatkuli, A., Pizarro, D., Bartoli, A.:
\newblock Non-rigid shape-from-motion for isometric surfaces using
  infinitesimal planarity.
\newblock In: BMVC. (2014)

\bibitem{White2007}
White, R., Crane, K., Forsyth, D.:
\newblock Capturing and animating occluded cloth.
\newblock In: SIGGRAPH. (2007)

\bibitem{Sundaram2010}
Sundaram, N., Brox, T., Keutzer, K.:
\newblock Dense point trajectories by gpu-accelerated large displacement
  optical flow.
\newblock In: ECCV. (2010)

\bibitem{Innmann2016}
Innmann, M., Zollh{\"o}fer, M., Nie{\ss}ner, M., Theobalt, C., Stamminger, M.
\newblock In: VolumeDeform: Real-Time Volumetric Non-rigid Reconstruction.
  Springer International Publishing, Cham (2016)  362--379

\bibitem{Gotardo2011}
Gotardo, P.F., Martinez, A.M.:
\newblock Computing smooth time trajectories for camera and deformable shape in
  structure from motion with occlusion.
\newblock IEEE Trans.\ on Pattern Analysis and Machine Intelligence
  \textbf{33}(10) (2011)  2051--2065

\bibitem{Chhatkuli2017}
Chhatkuli, A., Pizarro, D., Collins, T., Bartoli, A.:
\newblock Inextensible non-rigid structure-from-motion by second-order cone
  programming.
\newblock IEEE Transactions on Pattern Analysis and Machine Intelligence
  \textbf{PP}(99) (2017)  1--1

\end{thebibliography}
}

\end{document}